\documentclass[11pt]{article}

\usepackage{titlesec}

\usepackage{amsfonts,amsmath,amssymb,amsthm,color}
\usepackage[margin=1.0in]{geometry}
\usepackage{microtype}
\definecolor{weborange}{rgb}{.8,.3,.3}
\definecolor{webblue}{rgb}{0,0,.8}
\definecolor{internallinkcolor}{rgb}{0,.5,0}
\definecolor{externallinkcolor}{rgb}{0,0,.5}

\providecommand{\remove}[1]{}

\usepackage{amsthm} % in order to remove warnings such as "destination with the same identifier (name{corollary.1.3}) has been already used, duplicate ignored" the order of packages should be hyperref, amsthm, cleveref

\usepackage{comment}

\usepackage{enumerate}
\usepackage[labelfont=bf]{caption}

\usepackage[%pagebackref,
hyperfootnotes=false,
colorlinks=true,
urlcolor=externallinkcolor,
linkcolor=internallinkcolor,
filecolor=externallinkcolor,
citecolor=internallinkcolor,
breaklinks=true,
pdfstartview=FitH,
pdfpagelayout=OneColumn]{hyperref}

\usepackage[backend=bibtex,style=ieee-alphabetic,natbib=true,backref=true]{biblatex} %added
\usepackage{cleveref}
\usepackage{xspace}
\usepackage{dsfont}

\usepackage{colortbl}
\usepackage{xcolor}
\usepackage{graphicx}

\providecommand{\remove}[1]{}

\ifdefined\IsDraft
\newcommand{\authnote}[2]{{\bf [{\color{red} #1's Note:} {\color{blue} #2}]}}
\else
\newcommand{\authnote}[2]{}
\fi

\remove{
	
	\titleclass{\subsubsubsection}{straight}[\subsection]
	
	\newcounter{subsubsubsection}[subsubsection]
	\renewcommand\thesubsubsubsection{\thesubsubsection.\arabic{subsubsubsection}}
	\titleformat{\subsubsubsection}
	{\normalfont\normalsize\bfseries}{\thesubsubsubsection}{1em}{}
	\titlespacing*{\subsubsubsection}
	{0pt}{3.25ex plus 1ex minus .2ex}{1.5ex plus .2ex}
	
	\makeatletter
	\renewcommand\paragraph{\@startsection{paragraph}{5}{\z@}%
		{3.25ex \@plus1ex \@minus.2ex}%
		{-1em}%
		{\normalfont\normalsize\bfseries}}
	\renewcommand\subparagraph{\@startsection{subparagraph}{6}{\parindent}%
		{3.25ex \@plus1ex \@minus .2ex}%
		{-1em}%
		{\normalfont\normalsize\bfseries}}
	\def\toclevel@subsubsubsection{4}
	\def\toclevel@paragraph{5}
	\def\toclevel@paragraph{6}
	\def\l@subsubsubsection{\@dottedtocline{4}{7em}{4em}}
	\def\l@paragraph{\@dottedtocline{5}{10em}{5em}}
	\def\l@subparagraph{\@dottedtocline{6}{14em}{6em}}
	\makeatother
	
	\setcounter{secnumdepth}{4}
	\setcounter{tocdepth}{4}
} % remove

\newenvironment{algorithm}{\begin{mybox} \vspace{-.1in}\begin{algo}}{ \vspace{-.1in} \end{algo}\end{mybox}}

\newenvironment{mybox}{\begin{center}\begin{tabular}{|p{0.97\linewidth}|c|}   \hline} {  \\ \hline \end{tabular} \end{center}}			

\let\originalleft\left
\let\originalright\right
\renewcommand{\left}{\mathopen{}\mathclose\bgroup\originalleft}
\renewcommand{\right}{\aftergroup\egroup\originalright}

	\newcommand{\aka} {also known as,\xspace}

	\newcommand{\abs}[1]{\left\lvert #1 \right\rvert}
	\newcommand{\ceil}[1]{\left\lceil #1 \right\rceil}
	
	\newcommand{\signn}[1]{\sign\paren{#1}}
	\newcommand{\ip}[1]{\iprod{#1}}
	\newcommand{\iprod}[1]{\langle #1 \rangle}
	
	\newcommand{\set}[1]{\ens{#1}}

	\newcommand{\paren}[1]{\left(#1\right)}

	\newcommand{\floor}[1]{\left \lfloor#1 \right \rfloor}

	\newcommand{\norm}[1]{\left\lVert#1\right\rVert}

	\newcommand{\eqdef}{:=}

	\newcommand{\N}{{\mathbb{N}}}

	\newcommand{\zo}{\set{0,1}}
	\newcommand{\oo}{\set{-1,1}}

	\newcommand{\condition}{\;\ifnum\currentgrouptype=16 \middle\fi|\;}

	\newcommand{\eps}{\varepsilon}

	\newcommand{\la}{\gets}

	\newcommand{\argmin}{\operatorname*{argmin}}

	\newcommand{\MathAlg}[1]{\mathsf{#1}}
	
	\newcommand{\sign}{\MathAlg{sign}}

	\newcommand{\Ensuremath}[1]{\ensuremath{#1}\xspace}

	\newcommand{\tth}[1]{\Ensuremath{#1^{\rm th}}}
	
	\newcommand{\ith}{\tth{i}}
	\newcommand{\jth}{\tth{j}}

	\renewcommand{\cref}{\Cref}
	\usepackage{aliascnt}
	
	\newtheorem{theorem}{Theorem}[section]
	
	\newaliascnt{lemma}{theorem}
	\newtheorem{lemma}[lemma]{Lemma}
	\aliascntresetthe{lemma}
	\crefname{lemma}{Lemma}{Lemmas}

	\newaliascnt{observation}{theorem}
	\newtheorem{observation}[observation]{Observation}
	\aliascntresetthe{observation}
	\crefname{observation}{Observation}{Observation}

	\newaliascnt{claim}{theorem}
	\newtheorem{claim}[claim]{Claim}
	\aliascntresetthe{claim}
	\crefname{claim}{Claim}{Claims}
	
	\newaliascnt{corollary}{theorem}
	
	\aliascntresetthe{corollary}
	\crefname{corollary}{Corollary}{Corollaries}

	\newaliascnt{construction}{theorem}
	
	\aliascntresetthe{construction}
	\crefname{construction}{Construction}{Constructions}
	
	\newaliascnt{fact}{theorem}
	\newtheorem{fact}[fact]{Fact}
	\aliascntresetthe{fact}
	\crefname{fact}{Fact}{Facts}
	
	\newaliascnt{proposition}{theorem}
	\newtheorem{proposition}[proposition]{Proposition}
	\aliascntresetthe{proposition}
	\crefname{proposition}{Proposition}{Propositions}
	
	\newaliascnt{conjecture}{theorem}
	
	\aliascntresetthe{conjecture}
	\crefname{conjecture}{Conjecture}{Conjectures}

	\newaliascnt{definition}{theorem}
	\newtheorem{definition}[definition]{Definition}
	\aliascntresetthe{definition}
	\crefname{definition}{Definition}{Definitions}
	
	\newaliascnt{notation}{theorem}
	
	\aliascntresetthe{notation}
	\crefname{notation}{Notation}{Notation}
	
	\newaliascnt{assertion}{theorem}
	
	\aliascntresetthe{assertion}
	\crefname{assertion}{Assertion}{Assertion}
	
	\newaliascnt{assumption}{theorem}
	
	\aliascntresetthe{assumption}
	\crefname{assumption}{Assumption}{Assumption}

	\newaliascnt{remark}{theorem}
	\newtheorem{remark}[remark]{Remark}
	\aliascntresetthe{remark}
	\crefname{remark}{Remark}{Remarks}
	
	\newaliascnt{question}{theorem}
	\newtheorem{question}[question]{Question}
	\aliascntresetthe{question}
	\crefname{question}{Question}{Question}

	\newaliascnt{example}{theorem}
	\ifdefined\excludeexample
	\else
	
	\fi

	\aliascntresetthe{example}
	\crefname{exmaple}{Example}{Examples}

	\crefname{equation}{Equation}{Equations}

	\newaliascnt{proto}{theorem}
	
	\newtheorem{proto}[proto]{Protocol}
	
	\aliascntresetthe{proto}
	\crefname{proto}{protocol}{protocols}
	\newaliascnt{algo}{theorem}
	\newtheorem{algo}[algo]{Algorithm}
	\aliascntresetthe{algo}
	\crefname{algo}{algorithm}{algorithms}

	\newaliascnt{func}{theorem}
	\newtheorem{func}[func]{Function}
	\aliascntresetthe{func}
	\crefname{func}{function}{functions}

	\newaliascnt{expr}{theorem}
	\newtheorem{expr}[expr]{Experiment}
	\aliascntresetthe{expr}
	\crefname{experiment}{experiment}{experiments}

	\newaliascnt{gm}{theorem}
	\newtheorem{gm}[gm]{Game}
	\aliascntresetthe{gm}
	\crefname{game}{game}{games}

	\newcommand{\stepref}[1]{Step~\ref{#1}}

	\def\FullBox{$\Box$}
	\def\qed{\ifmmode\qquad\FullBox\else{\unskip\nobreak\hfil
			\penalty50\hskip1em\null\nobreak\hfil\FullBox
			\parfillskip=0pt\finalhyphendemerits=0\endgraf}\fi}
	
	\def\qedsketch{\ifmmode\Box\else{\unskip\nobreak\hfil
			\penalty50\hskip1em\null\nobreak\hfil$\Box$
			\parfillskip=0pt\finalhyphendemerits=0\endgraf}\fi}

	{\proof[#1]\list{}{\leftmargin=1cm\rightmargin=1cm}}
	{\endproof%
		\vspace{10pt}
	}
	
	\newcommand{\eex}[2]{\Ex_{#1}\left[#2\right]}
	\newcommand{\ex}[1]{\Ex\left[#1\right]}
	\newcommand{\Ex}{{\mathrm E}}
	\renewcommand{\Pr}{{\mathrm {Pr}}}
	\newcommand{\pr}[1]{\Pr\left[#1\right]}
	\newcommand{\ppr}[2]{\Pr_{#1}\left[#2\right]}

	\newcommand{\Ac}{\mathsf{A}}
	\newcommand{\PAP}{\mathsf{PAP}}

	\newcommand{\Fc}{\mathsf{F}}
	\newcommand{\Gc}{\mathsf{G}}

	\newcommand{\Mc}{\mathsf{M}}

	\newcommand{\Bc}{\mathsf{B}}

	\newcommand{\ens}[1]{\{#1\}}
	\newcommand{\size}[1]{\left|#1\right|}

	\newcommand{\tP}{\widetilde{P}}

	\newcommand{\cW}{{\cal{W}}}

	\def\cD{{\cal D}}
	\def\cE{{\cal E}}

	\def\cH{{\cal H}}
	\def\cI{{\cal I}}
	\def\cJ{{\cal J}}
	\def\cK{{\cal K}}

	\def\cN{{\cal N}}
	
	\def\cP{{\cal P}}

	\def\cS{{\cal S}}
	\def\cT{{\cal T}}

	\def\cW{{\cal W}}
	\def\cX{{\cal X}}
	
	\def\cY{{\cal Y}}
	\def\cZ{{\cal Z}}
	\def\bbI{{\mathbb I}}
	
	\def\bbN{{\mathbb N}}
	
	\def\bbR{{\mathbb R}}

	\newcommand{\Tableofcontents}{
		\thispagestyle{empty}
		\pagenumbering{gobble}
		\clearpage
		\tableofcontents
		\thispagestyle{empty}
		\clearpage
		\pagenumbering{arabic}
	}

	\newcommand{\indic}[1]{\mathds{1}{\set{#1}}}

	\newcommand{\pt}[1]{\boldsymbol{#1}}
	\newcommand{\px} {\pt{x}}

	\newcommand{\y} {\pt{y}}
	\newcommand{\pz} {\pt{z}}

	\newcommand{\pv} {\pt{v}}

	\newcommand{\tv} {\tilde{v}}
	\newcommand{\hPi} {\widehat{\Pi}}
	\newcommand{\tPi} {\widetilde{\Pi}}

	\newcommand{\hU} {\widehat{U}}
	\newcommand{\hV} {\widehat{V}}

	\newcommand{\bq}{\bold{q}}
	\newcommand{\bs}{\bold{s}}
	\newcommand{\by}{\bold{y}}
	\newcommand{\bY}{\bold{Y}}
	\newcommand{\bb}{\bold{b}}
	\newcommand{\bp}{\bold{p}}
	\newcommand{\bP}{\bold{P}}
	\newcommand{\bJ}{\bold{J}}
	
	\newcommand{\btP}{\bold{\tP}}
	\newcommand{\bX}{\bold{X}}
	\newcommand{\bE}{\bold{E}}
	
	\newcommand{\bZ}{\bold{Z}}
	\newcommand{\ba}{\bold{a}}
	\newcommand{\bi}{\bold{i}}
	\newcommand{\bA}{\bold{A}}
	\newcommand{\bx}{\bold{x}}
	
	\newcommand{\bB}{\bold{B}}

	\newcommand{\bv}{\bold{v}}
	\newcommand{\bN}{\bold{N}}
	\newcommand{\bw}{\bold{w}}
	\newcommand{\bt}{\bold{t}}
	\newcommand{\bu}{\bold{u}}
	\newcommand{\bz}{\bold{z}}
	\newcommand{\bF}{\bold{F}}
	
	\newcommand{\bPi}{\bold{\Pi}}
	\newcommand{\btPi}{\bold{\tPi}}
	\newcommand{\bhPi}{\bold{\hPi}}
	
	\newcommand{\btv}{\bold{\tv}}

	\newcommand{\Lap}{{\rm Lap}}

	\newcommand{\EstSubspace}{\mathsf{EstSubspace}}
	\newcommand{\Agg}{\mathsf{Agg}}
	\newcommand{\NaiveAgg}{\mathsf{Naive\_Agg}}
	\newcommand{\SSAgg}{\mathsf{SS\_Agg}}
	\newcommand{\AlgFriendlyAvg}{\mathsf{FriendlyAvg}}

	\newcommand{\FCAverage}{\mathsf{FC\_Average}}

	\makeatletter
	\let\xx@thm\@thm
	\AtBeginDocument{\let\@thm\xx@thm}
	\makeatother

	\newcommand{\hSigma}{\widehat{\Sigma}}
	
	\newcommand{\HG}{\mathcal{HG}}

        \renewcommand{\epsilon}{\varepsilon}

\newcommand{\ignore}[1]{}

\title{On Differentially Private Subspace Estimation in a Distribution-Free Setting}
\ifdefined\IsAnonymous
\author{Anonymized for Submission}
\else
\author{
Eliad Tsfadia\thanks{Department of Computer Science, Georgetown University. E-mail: \texttt{eliadtsfadia@gmail.com}.  Work supported by a gift to Georgetown University.}
}
\fi
\begin{document}

\maketitle

\begin{abstract}
	
Private data analysis faces a significant challenge known as the curse of dimensionality, leading to increased costs. However, many datasets possess an inherent low-dimensional structure. For instance, during optimization via gradient descent, the gradients frequently reside near a low-dimensional subspace.
If the low-dimensional structure could be privately identified using a small amount of points, we could avoid paying for the high ambient dimension.

On the negative side, \citet{DTTZ14} proved that privately estimating subspaces, in general, requires an amount of points that has a polynomial dependency on the dimension. However, their bound do not rule out the possibility to reduce the number of points for ``easy'' instances. Yet, providing a measure that captures how much a given dataset is ``easy'' for this task turns out to be challenging, and was not properly addressed in prior works.

Inspired by the work of \citet{SS21}, we provide the first measures that quantify ``easiness'' as a function of multiplicative singular-value gaps in the input dataset, and support them with new upper and lower bounds. In particular, our results determine the first type of gaps that are sufficient and necessary for privately estimating a subspace with an amount of points that is independent of the dimension. Furthermore, we realize our upper bounds using a practical algorithm and demonstrate its advantage in high-dimensional regimes compared to prior approaches.

\end{abstract}

\Tableofcontents

\section{Introduction}

Differentially private (DP) \cite{DMNS06} algorithms typically exhibit a significant dependence on the dimensionality of their input, as their error or sample complexity tends to grow polynomially with the dimension. 
This cost of dimensionality is inherent in many problems, as \cite{BUV14,SU17,DworkSSUV15} showed that any method that achieves lower error rates is vulnerable to tracing attacks (\aka membership inference attacks).
However, these lower bounds consider algorithms that guarantee accuracy for worst-case inputs and do not rule out the possibility of achieving higher accuracy for ``easy'' instances.

\paragraph{Example: DP averaging.} As a simple prototypical example, consider the task of DP averaging. In this task, the input dataset consists of $d$-dimensional points $x_1,\dots,x_n \in \mathbb{R}^d$, and the goal is to estimate their average $\frac{1}{n}\sum_{i=1}^n x_i$ using a DP algorithm while minimizing the $\ell_2$ additive error. One natural way to capture input ``easiness'' for this task is via the maximal $\ell_2$ distance between any two points (i.e., points that are closer to each other are considered ``easier''). 
Indeed, \cite{FriendlyCore22,PTU24} showed that if the points are $\gamma$-close to each other, and we aim for an accuracy of $\lambda \gamma$ (i.e., an accuracy that is \emph{proportional} to the ``easiness'' parameter $\gamma$), then it is sufficient and necessary to use $n = \tilde{\Theta}(\sqrt{d}/\lambda)$ points. Equivalently, if we aim for an accuracy of $\alpha$, then by applying these results with $\lambda = \alpha/\gamma$, we obtain that the answer is $n = \tilde{\Theta}(\gamma \sqrt{d}/\alpha)$.
This, in particular, implies that when $\gamma \leq \alpha/\sqrt{d}$ (i.e., the points are very close to each other), then $\tilde{O}(1)$ points are sufficient, but for $\gamma = \alpha/d^{1/2-\Omega(1)}$, a polynomial dependency on $d$ is necessary in general.

\paragraph{DP subspace estimation.} In this work, we consider the more complex problem of DP subspace estimation: Given a dataset $X=(x_1,\dots,x_n) \in (\mathbb{R}^d)^n$ of unit norm points and a parameter $k$, estimate the top-$k$ subspace of $Span\set{x_1,\ldots,x_n}$. 
The main goal of this work is to answer the following meta question:

\begin{question}\label{question:meta}
	How should we quantify how ``easy'' a given dataset is for DP subspace estimation?
\end{question}

Since the dimension $d$ is very large in many settings, we aim at providing tight measures that smoothly eliminate the dependency on $d$ as a function of input ``easiness''. In particular, we want to be able to identify when we can avoid paying on the ambient dimension $d$, and when a polynomial dependency on $d$ is unavoidable.

\subsection{Motivation: DP-SGD}\label{sec:intro:motivation}

To motivate the problem, consider the task of privately training large neural networks.  The most commonly used tool to perform such a private training is the differentially-private stochastic gradient descent (DP-SGD) \cite{AbadiDPSGD16,BassilyDPSGD14,SongDPSGD13} – a private variant of SGD that perturbs each gradient update with random noise vector drawn from an isotropic Gaussian distribution. However, this approach does not differentiate between ``easy'' gradients and ``hard'' ones, which results with high error rates when the ambient dimension - the number of parameters in the model - is large. However, empirical evidence and theoretical analysis indicate that while training some deep learning models, the gradients tend to live near a low-dimensional subspace \cite{ACGMMTZ16,LXTSG18,GARD18,LFLY18,LGZCB20,ZWB21,FT20,LLHplus22,GAWplus22,KRRT20}. In particular, \cite{GARD18} showed that in some cases, the low dimension is the number of classes in the dataset, and the gradients tend to be close and well-spread inside this subspace. If we could exploit such a low-dimensional structure into an (inexpensive) private and useful projection matrix, we could reduce the error of DP-SGD by making it dependent solely on the low dimension.

We start by defining the setting of DP subspace estimation more formally.

\subsection{Subspace Estimation}

We consider the setting of \cite{DTTZ14}. That is, the input dataset consists of $n$ points of unit norm $x_1,\ldots,x_n \in \cS_d \eqdef \set{v \in \bbR^d \colon \norm{v}_2 =1 }$ and a parameter $k$, and the goal is to output a $k$-dimensional projection matrix $\Pi$ such that $\Pi \cdot X^T$ is ``close'' to $X^T$ as possible, where $X$ denotes the $n\times d$ matrix whose rows are $x_1,\ldots,x_n$. We measure the accuracy of our estimation using the ``usefulness'' definition of \cite{DTTZ14}:

\begin{definition}[$\alpha$-useful]\label{def:intro:useful}
	We say that a rank-$k$ projection matrix $\Pi$ is $\alpha$-useful for a matrix $X \in (\cS_d)^n$ if for any $k$-rank projection matrix $\Pi'$:
	\begin{align*}
		\norm{\Pi \cdot X^T}_F^2 \geq \norm{\Pi' \cdot X^T}_F^2 - \alpha \cdot n,
	\end{align*}
	where $\norm{\cdot}_F$ denotes the Frobenius norm.\footnote{The Frobenius norm of a matrix $A = (a_i^j)_{i \in [n], j \in [d]}$ is defined by $\norm{A}_F = \sqrt{\sum_{i\in [n], j \in [d]}\paren{a_i^j}^2}$.}
\end{definition}

Observe that any projection matrix is $1$-useful for any $X$ (because  $\norm{X}_F^2 =  \sum_{i=1}^n \norm{x_i}_2^2 = n$). Therefore, we will be interested in smaller values of $\alpha$ (e.g., $0.001$).

\subsection{Prior Works}\label{sec:intro:prior_work}

Without privacy restrictions, we can find a $0$-useful (i.e., optimal) solution using Singular-Value Decomposition (SVD). The SVD of $X$ is $X = U \Sigma V^T$, where $U\in \mathbb{R}^{n\times n}$ and $V \in \mathbb{R}^{d \times d}$ are unitary matrices, and $\Sigma$ is an $n \times d$ diagonal matrix which has values $\sigma_1 \geq \ldots \geq \sigma_{\min\set{n,d}} \geq 0$ along the diagonal. The top-$k$ rows subspace of $X$ is given by the span of the first $k$ columns of $V$, and it can be computed, e.g., by applying Principal Component Analysis (PCA) on the covariance matrix $A = X^T X$ (the eigenvectors of $A$ are the columns of $V$). 

With differential privacy, however, the problem is much harder, and \cite{DTTZ14} showed a lower bound of $n\geq \tilde{\Omega}(k \sqrt{d})$ for computing a $0.001$-useful $k$-rank projection matrix under $(1, \Omega(1/n^2))$-DP. This bound, however, only holds for algorithms that provide accuracy for worst-case instances and does not rule out the possibility of achieving high accuracy with smaller values of $n$ for input points that are very close to being in a $k$-dimensional subspace (i.e., ``easy" instances). 

Perhaps the easiest instances are those that \emph{exactly} lie in a $k$-dimensional subspace, and are well-spread within it (i.e., there is no $(k-1)$-dimensional subspace that contains many points). Indeed, \cite{SS21} and \cite{AL22} developed $(\eps,\delta)$-DP algorithms for such instances that precisely recover the subspace using only $n = \tilde{\Theta}_{\eps, \delta}(k)$ points. However, while these algorithms are robust to changing a few points, they are very brittle if we change all the points by a little bit.

One approach to smoothly quantify how much a dataset is ``easy'' is to consider the \emph{additive-gap} $\sigma_k^2 - \sigma_{k+1}^2$. Indeed,  \cite{DTTZ14, GGB18} present $(\eps,\delta)$-DP algorithms that output $0.001$-useful projection using $n  = \tilde{\Theta}_{\eps,\delta}\paren{\frac{k \sqrt{d}}{\sigma_k^2 - \sigma_{k+1}^2}}$ points. Yet, the downside of such additive-gap based approaches is their inherent dependency on the dimension $d$. Even in the extreme case where the points exactly lie in a $k$-dimensional subspace and well-spread within it, the additive gap $\sigma_k^2 - \sigma_{k+1}^2$ is at most $n/k$, which still results with a polynomial dependency on $d$.

The only existing approach to eliminate the dependency on $d$ in some non-exact cases is the one of \cite{SS21} (their ``approximate'' case). Rather than quantifying easiness as a function of the input dataset, they consider a setting where the points are sampled i.i.d. from some distribution, and implicitly measure how ``easy'' a distribution is according to some stability notion.
In particular, they show that a $d$-dimensional Gaussian $\mathcal{N}(\vec{0}, \Sigma)$ with a \emph{multiplicative-gap} $\frac{\sigma_{k+1}(\Sigma)}{\sigma_k(\Sigma)} \leq \tilde{\Theta}_{\eps,\delta,k}\left(\frac{1}{d^2}\right)$ is ``stable'' enough for estimating the top-$k$ subspace of $\Sigma$ with sample complexity that is independent of $d$.
While \cite{SS21} do not provide an answer to \cref{question:meta}, they inspired our work to consider multiplicative singular-value gaps  \emph{in the input dataset} as a measure for easiness.

\subsection{Defining Subspace Estimators}\label{sec:intro:def-subspace-est}

Towards answering \cref{question:meta}, we consider mechanisms $\Mc$ that are parameterized by $k$, $\lambda$, and $\beta$, and satisfy the following utility guarantee: Given a dataset $X = (x_1,\ldots,x_n) \in (\cS_d)^n$ and a value $\gamma$ as inputs, such that $X$ is ``$\gamma$-easy'' for $k$-subspace estimation, then with probability at least $\beta$ over a random execution of $\Mc(X, \gamma)$, the output $\Pi$ is an $\lambda \gamma$-useful rank-$k$ projection matrix for $X$.\footnote{Similarly to the DP averaging example, we consider algorithms which guarantee accuracy that is proportional to the ``easiness'' parameter $\gamma$, and we measure the ``quality'' of the estimations by the parameter $\lambda$.} In \cref{def:intro:subspace-est}, the ``$\gamma$-easy" property is abstracted by a predicate $f$. We also allow an additional parameter $\gamma_{\max}$ to relax the utility for non-easy instances (i.e., we would not require a utility guarantee for instances that are not ``$\gamma_{\max}$-easy''). Furthermore, we only focus on cases in which $X$ has at least $k$ significant directions, which is formalized by requiring that $\sigma_{k}^2(X) \geq 0.01 \cdot n/k$ (we refer to \cref{remark:large_sigma_k} for how our upper bounds, stated in \cref{thm:intro:upper_bound:weak}, can handle smaller values of $\sigma_k$ using an additional parameter or different assumptions).

\begin{definition}[$(k,\lambda, \beta, \gamma_{\max}, f)$-Subspace Estimator]\label{def:intro:subspace-est}
	Let $n,k,d \in \bbN$ s.t. $k \leq \min\set{n,d}$,  let $\beta \in (0,1]$, and let $f\colon (\cS_d)^n \times [0,1] \rightarrow \zo$ be a predicate.  We say that $\Mc \colon (\cS_d)^n \times [0,1] \rightarrow \bbR^{d \times d}$ is an $(k, \lambda, \beta, \gamma_{\max}, f)$-subspace estimator, if for every $X \in (\cS_d)^{n}$ and $\gamma \leq \gamma_{\max}$ with $\sigma_{k}^2(X) \geq 0.01 \cdot n/k$ and $f(X,\gamma) = 1$, it holds that
	\begin{align*}
		\ppr{\Pi \sim \Mc(X, \gamma)}{\Pi\text{ is }(\lambda \gamma)\text{-useful for }X} \geq \beta.
	\end{align*}
\end{definition}

\subsection{Quantifying Easiness -  Our Approach}\label{sec:intro:our-approach}

In this work, we develop two types of smooth measures (captured by the predicate $f$ in \cref{def:intro:subspace-est}) for input ``easiness'', which are translated to the following two types of subspace estimators:

\begin{definition}[$(k, \lambda, \beta, \gamma_{\max})$-Weak Subspace Estimator]\label{def:intro:weak-subspace-est}
	$\Mc \colon (\cS_d)^n \times [0,1] \rightarrow \bbR^{d \times d}$ is called an \emph{$(k, \lambda, \beta, \gamma_{\max})$-Weak Subspace Estimator}, if it is $(k, \lambda, \beta, \gamma_{\max}, f)$-Subspace Estimator for the predicate $f(X,\gamma)$ that outputs $1$ iff $\sum_{i=k+1}^n \sigma_{i}^2(X) \leq \gamma^2 \sigma_k^2(X)$.
\end{definition}

\begin{definition}[$(k, \lambda, \beta, \gamma_{\max})$-Strong Subspace Estimator]\label{def:intro:strong-subspace-est}
	$\Mc \colon (\cS_d)^n \times [0,1] \rightarrow \bbR^{d \times d}$ is called an \emph{$(k, \lambda, \beta, \gamma_{\max})$-Strong Subspace Estimator}, if it is $(k, \lambda, \beta, \gamma_{\max}, f)$-Subspace Estimator for the predicate $f(X,\gamma)$ that outputs $1$ iff $\sigma_{k+1}(X) \leq \gamma \sigma_k(X)$.
\end{definition}

In both cases, we define $f$ based on multiplicative singular-value gaps in the input dataset, but the difference is what type of gap the value $\gamma$ bounds: Strong estimators depend solely on the gap $\sigma_{k+1}/\sigma_k$ without taking into account smaller singular values. Weak estimators, on the other, depend on the gap $\sqrt{\sum_{i=k+1}^{\min\set{n,d}} \sigma_i^2}/\sigma_k$. Note that a strong estimator is, in particular, a weak one (with the same parameters).  Also note that both measures smoothly converge to the exact $k$-subspace case: When each gap tends to zero, the points tend to be closer to a $k$-dimensional subspace.

We provide new upper and lower bounds for both types of estimators.

\subsubsection{Our Upper Bounds}

\begin{theorem}[Weak estimator]\label{thm:intro:upper_bound:weak}
	There exists an $(k, \lambda, \beta = 0.9, \gamma_{\max} = \Theta(\min\set{\frac1{\lambda}, 1}))$-weak subspace estimator $\Mc \colon (\cS_d)^n \times [0,1] \rightarrow \bbR^{d \times d}$ with $n = \tilde{O}_{\eps,\delta}\paren{k + \frac{\min\set{k^2 \sqrt{d}, \: kd}}{\lambda}}$ such that $\Mc(\cdot, \gamma)$ is $(\eps,\delta)$-DP for every $\gamma \in [0,1]$.
\end{theorem}

\begin{theorem}[Strong estimator]\label{thm:intro:upper_bound:strong}
	There exists an $(k, \lambda, \beta = 0.8, \gamma_{\max} = \tilde{\Theta}(\min \set{\frac1{\lambda}, \: \frac{\lambda^2}{k^2d}}))$-strong subspace estimator $\Mc \colon (\cS_d)^n \times [0,1] \rightarrow \bbR^{d \times d}$ with $n = \widetilde{O}_{\eps,\delta}\paren{k + \frac{k^3 d}{\lambda^2}}$ such that $\Mc(\cdot, \gamma)$ is $(\eps,\delta)$-DP for every $\gamma \in [0,1]$.
\end{theorem}

Both of our estimators provide a useful projection by outputting a matrix that is close (in Frobenius norm) to the projection matrix of the top-$k$ rows subspace.
Their running time is $\frac{n}{m} \cdot T(m, d, k) + \tilde{O}(d k n)$ for some $m = \tilde{\Theta}(k)$, where $T(m,d,k)$ denotes the running time required to compute (non-privately) a projection matrix to the top-$k$ rows subspace of an $m \times d$ matrix.
We refer to \cref{sec:techniques:UB} for a detailed overview.

For simplifying the presentation and the formal utility guarantees, we assume that our algorithms know the values of $\gamma$ (the bound on the multiplicative-gap) and of $k$ beforehand. Yet, we show that both assumptions are not inherent, and we refer to \cref{remark:unknown_gamma_k} for additional details. 

We also remark that in both theorems, it is possible to increase the confidence $\beta$ to any constant smaller than $1$ without changing the asymptotic cost.

\subsubsection{Our Lower Bounds}

\begin{theorem}[Lower bound for weak estimators]\label{thm:intro:lower_bound:weak}
	If $\Mc \colon (\cS_d)^n \times [0,1] \rightarrow \bbR^{d \times d}$ is a \emph{$(k, \lambda, \beta= 0.1, \gamma_{\max} = \Theta(\frac1{\lambda}))$-weak subspace estimator} for $1\leq \lambda \leq \Theta(\frac{d}{k \log k})$ and $\Mc(\cdot, \gamma)$ is $\paren{1,\frac{1}{50nk}}$-DP for every $\gamma \in [0,1]$, then $n \geq \tilde{\Omega}\paren{\frac{\sqrt{k d}}{\lambda}}$.
\end{theorem}

\begin{theorem}[Lower bound for strong estimators]\label{thm:intro:lower_bound:strong}
	If $\Mc \colon (\cS_d)^n \times [0,1] \rightarrow \bbR^{d \times d}$ is a \emph{$(k, \lambda, \beta = 0.1, \gamma_{\max} = \Theta\paren{\frac1{\lambda}})$-strong subspace estimator} for $1\leq \lambda \leq \Theta(\frac{d}{\log k})$ and $\Mc(\cdot, \gamma)$ is $\paren{1,\frac{1}{50nk}}$-DP for every $\gamma$, then $n \geq \tilde{\Omega}\paren{\frac{k\sqrt{d}}{\lambda}}$.
\end{theorem}

Our lower bounds are more technically involved, and use a novel combination of generating hard-instances using the tools from \cite{PTU24} for proving smooth lower bounds, and extracting sensitive vectors from useful projection matrices using ideas from the lower bound of  \cite{DTTZ14}.  Both lower bounds are proven by generating hard-instances that are ``$\gamma$-easy'' for $\gamma = \frac1{1000 \lambda}$.
We refer to \cref{sec:techniques:LB} for a detailed overview.

We remark that \cite{PTU24} recently proved a similar lower bound for the special case of $k=1$ (estimating the top-singular vector). However, their result strongly relies on the similarity between averaging and estimating top-singular vector in their hard instances, which does not hold for the case $k \geq 2$.

\cref{table:bounds} summarizes our bounds for $k \leq \sqrt{d}$.

\begin{table}[]
	\centering
	\begin{tabular}{c|c|c|}
		\cline{2-3}
		\multicolumn{1}{l|}{}             & \multicolumn{1}{l|}{Weak Estimator} & \multicolumn{1}{l|}{Strong Estimator} \\ \hline
		\multicolumn{1}{|c|}{Upper Bound} & $\tilde{O}_{\eps,\delta}\paren{k + \frac{k^2 \sqrt{d}}{\lambda}}$           & $\widetilde{O}_{\eps,\delta}\paren{k + \frac{k^3 d}{\lambda^2}}$             \\ \hline
		\multicolumn{1}{|c|}{Lower Bound} & $\tilde{\Omega}\paren{\frac{\sqrt{k d}}{\lambda}}$          & $\tilde{\Omega}\paren{\frac{k\sqrt{d}}{\lambda}}$             \\ \hline
	\end{tabular}
	\caption{Our bounds on $n$ for subspace estimation (ignoring restrictions on $\gamma_{\max}$ and $\lambda$).}
	\label{table:bounds}
\end{table}

\subsubsection{Implications}

We offer two formulations which have the property we seek: If we aim for an error $\alpha$, and the dataset is ``$\gamma$-easy'' for a very small $\gamma$, we take $\lambda = \alpha/\gamma$ to reduce the number of necessary and sufficient points.

For strong estimators, the rate $n = n(\lambda)$ in \cref{thm:intro:upper_bound:strong} does not match the corresponding lower bound \cref{thm:intro:lower_bound:strong}, and \cref{thm:intro:upper_bound:strong} is limited to small values of $\gamma$.\footnote{If we take $\lambda = \alpha/\gamma$ (i.e., aiming for an error $\alpha$), the utility restriction on $\gamma_{\max}$ in  \cref{thm:intro:upper_bound:strong} implies that $\gamma$ should be smaller than $\Theta\paren{\frac{\alpha^{2/3}}{k^{2/3} d^{1/3}}}$. \label{small_gamma_max}} Yet, for small values of $k$, the strong-estimator bounds do imply that in order to privately compute an $0.001$-useful rank-$k$ projection with number of points that is independent of $d$, it is sufficient and necessary to have a gap $\sigma_{k+1}/\sigma_k$ of at most $\propto 1/\sqrt{d}$.  \cref{table:bounds:gamma} summarizes our bounds for the special case of outputing $0.001$-useful projection matrix, using our two different types of input ``easiness''.

\begin{table}[]
	\centering
	\begin{tabular}{c|c|c|}
		\cline{2-3}
		\multicolumn{1}{l|}{}             & \multicolumn{1}{l|}{Weak Estimator} & \multicolumn{1}{l|}{Strong Estimator} \\ \hline
		\multicolumn{1}{|c|}{Upper Bound} & $\tilde{O}_{\eps,\delta}\paren{k + \gamma_1 k^2 \sqrt{d}}$           & $\widetilde{O}_{\eps,\delta}\paren{k + \gamma_2^2 k^3 d}$             \\ \hline
		\multicolumn{1}{|c|}{Lower Bound} & $\tilde{\Omega}\paren{\gamma_1\sqrt{k d}}$          & $\tilde{\Omega}\paren{\gamma_2 k\sqrt{d}}$             \\ \hline
	\end{tabular}
	\caption{Our bounds on $n$ for computing $0.001$-useful projection, where $\gamma_1 = \sqrt{\sum_{i=k+1}^{\min\set{n,d}} \sigma_i^2}/\sigma_k$ and $\gamma_2 = \sigma_{k+1}/\sigma_k$ denote here the ``easiness'' parameters for weak and strong estimators, respectively. As mentioned in \cref{small_gamma_max}, our upper bound for strong estimators only hold for $\gamma_2 \leq \Theta(k^{-2/3} d^{-1/3})$.}
	\label{table:bounds:gamma}
\end{table}

\subsection{Empirical Evaluation}

We believe that private subspace estimation of easy instances could find practical applications. Therefore, we made an effort to realize our upper bounds using a \emph{practical} algorithm. In \cref{sec:experiments} we empirically compared our method to the additive-gap based approach of \cite{DTTZ14} for the task of privately estimating the empirical mean of points that are close to a small dimensional subspace, demonstrating the advantage of our approach in high-dimensional regimes.

\subsection{Other Related Work}

A closely related line of work is on Private PCA. 
\cite{DTTZ14} consider the simple algorithm that adds independent Gaussian noise to each entry of the covariance matrix $A = \sum_{i=1}^n x_i^T x_i \in \mathbb{R}^{d\times d}$, and then performs analysis on the noisy matrix. This method, predating the development of differential privacy \cite{BDMN05}, was later analyzed under differential privacy by \cite{MM09} and \cite{CSS13}. This simple algorithm is versatile and several bounds are provided for the accuracy of the noisy PCA. The downside of this is that a polynomial dependence on the ambient dimension $d$ is inherent for any instances (including ``easy" ones). While this approach has a variant that improves the accuracy of estimating the top-$k$ subspace as a function of the additive gap $\sigma_k^2 - \sigma_{k+1}^2$, it does not prevent the polynomial dependency on the dimension $d$ even for very easy instances.

Techniques from dimensionality reduction have been applied by \cite{HR12} and \cite{ABU18} to privately compute a low-rank approximation to the input matrix $X$. Similarly, \cite{HR13} and \cite{HP14} utilize the power iteration method with noise injected at each step to compute low-rank approximations to $X$. Despite their effectiveness, these methods, relying on noise addition, require sample complexity to grow polynomially with the ambient dimension to achieve meaningful guarantees.

Another approach, employed by \cite{BBDS12} and \cite{She19}, involves approximating the covariance matrix $A$ using dimensionality reduction. They show that the
dimensionality reduction step itself provides a privacy guarantee (whereas the aforementioned results did not exploit this and relied on noise added at a later stage).

\cite{CSS12,KT13,WSCHT16} apply variants of the exponential mechanism \cite{MT07} to privately select a low-rank approximation to the covariance matrix $A$. This method is nontrivial to implement and analyse,
but it ultimately requires the sample complexity to grow polynomially in the ambient dimension.

\cite{GGB18} exploit smooth sensitivity \cite{NRS07} to release a lowrank approximation to the matrix $A$. This allows adding less noise than worst case sensitivity,
under an eigenvalue gap assumption. However, the sample complexity $n$ is polynomial in the dimension $d$.

Another related area involves estimating the parameters of unbounded Gaussians \cite{KSSU22,AL22,KMV22,FriendlyCore22}. Notably, \cite{KSSU22} used the subspace learning algorithm of \cite{SS21} to efficiently learn the covariance matrix.

A recent popular trend in DP learning is to utilize a few \emph{public} examples to enhance accuracy. 
This has led to methods for private ML which project the sensitive gradients
onto a subspace estimated from the public gradients. By using a small amount of i.i.d. public
data, \cite{ZWB21} demonstrate that this approach can improve the accuracy of differentially private
stochastic gradient descent in high-privacy regimes and achieve a dimension-independent error rate.
Similarly, \cite{YZCL21} proposed GEP, a method that utilizes public data to identify the most useful
information carried by gradients, and then splits and clips them separately.
These works underscore the importance of identifying the subspace of gradients in private ML.

\subsection{Limitations and Future Directions}

From a theoretical perspective, our work is the first to provide proper measures for how ``easy'' a given dataset is which smoothly eliminates the dependency on the dimension $d$. 
Yet, closing the gap between our upper and lower bounds is still left open.
Specifically, for weak estimator, there is still a gap of $k^{1.5}$ between \cref{thm:intro:upper_bound:weak,thm:intro:lower_bound:weak}. For strong estimators, the upper-bound rate $n = n(\lambda)$ (\cref{thm:intro:upper_bound:strong}) does not align with the one of the lower bound (\cref{thm:intro:lower_bound:strong}), and it is left open to relax the restriction on $\gamma_{\max}$. One possible reason for some of these gaps (especially the dependency in $k$) is that our upper bounds follow the approach of \cite{SS21} to estimate (under some matrix norm) the projection matrix to the top-$k$ rows subspace (we do it in Frobenius norm). While estimating the projection matrix itself provides, in particular, a useful solution (\cref{prop:F-dist-implies-estimator}), the opposite direction is not true in general, and it could be possible to reduce the sample complexity by focusing on $\alpha$-usefulness (\cref{def:intro:useful}) directly, or alternatively, providing stronger lower bounds for estimating the projection matrix.

From a more practical standpoint, we empirically demonstrate the advantage of our approach in high-dimension regimes when the data is very close to a low-dimensional structure, which is directly translated to an advantage in private mean estimation of such instances. The downside of our approach is that it requires the points to be very close to a $k$-dimensional structure in order to be effective, which might not be sufficient for typical training scenarios in deep learning. It would be intriguing to explore if there is a connection between training parameters (e.g., the network structure) to the phenomena of gradients that are close to a low-dimensional subspace (mentioned in \cref{sec:intro:motivation}). If we could boost this closeness to regimes where our method achieves high accuracy, we could generate drastically improved private models. On the other hand, if we cannot do it, then our lower bounds indicate that improving DP-SGD via private subspace estimation might not be the right approach, and we should focus on different approaches for this task.

\subsection{Paper Organization}

In \cref{sec:techniques} we present proof overview for our results.
Notations, definitions and general statements used throughout the paper are given in \cref{sec:prelim}.
Our upper bounds (\cref{thm:intro:upper_bound:weak,thm:intro:upper_bound:strong}) are proven in \cref{sec:UB}. Our lower bounds (\cref{thm:intro:lower_bound:weak,thm:intro:lower_bound:strong}) are proven in \cref{sec:LB}. The empirical evaluation is provided in \cref{sec:experiments}.

\section{Techniques}\label{sec:techniques}

In this section we provide proofs overview for our upper bounds and lower bounds.

\subsection{Upper Bounds}\label{sec:techniques:UB}

Both of our estimators (\cref{thm:intro:upper_bound:weak,thm:intro:upper_bound:strong}) follow the same structure, but with different parameters.
Similarly to \cite{SS21}, our algorithms follow the sample-and-aggregate approach of \cite{NRS07}. That is, 
given a dataset $X = (x_1,\ldots,x_n) \in (\cS_d)^n$, we partition the rows into $t$ subsets, compute (non-privately) a projection matrix to the top-$k$ rows subspace of each subset, and then privately aggregate the projection matrices $\Pi_1,\ldots,\Pi_t$.
For doing that, we need to argue that $\Pi_1,\ldots,\Pi_t$ are expected to be close to each other. In the Gaussian setting of \cite{SS21}, this holds by concentration properties of Gaussian distributions. In our setting, however, it is unreasonable to expect that arbitrary partitions will lead to similar subspaces. For instance, consider the matrix $X$ whose first $n/k$ rows are $e_1 = (1,0,\ldots,0)$, the next $n/k$ rows are $e_2 = (0,1,\ldots,0)$, and so forth until $e_k$. Even though $X$ is rank-$k$ and has $\sigma_1^2 = \ldots = \sigma_k^2 = n/k$,  if we simply partition the rows according to their order, then most of those subsets will induce a rank-$1$ matrix which clearly does not represent the original matrix $X$. Therefore, we must consider a more clever partition that will guarantee a good representation of the top-$k$ rows subspace of $X$ in each subset.

There is an extensive line of works who aim for methods to choose a small subset of rows that provides a good low-approximation for the original matrix (e.g., see \cite{Mahoney2011} for a survey). Yet, most of these methods are data-dependent, and therefore seem less applicable for privacy. 

In this work we show that by simply using a \emph{uniformly random} partition into subsets of size $\tilde{\Theta}(k)$, then w.h.p. each subset induces a projection matrix that is close to the projection matrix of the top-$k$ rows subspace of $X$:

\begin{lemma}\label{lemma:technique:Pi-distance-bounds}
	Let $X = (x_1,\ldots, x_n) \in (\cS_d)^n$ with singular values $\sigma_1 \geq \ldots \geq \sigma_{\min\set{n,d}}\geq 0$, and let $\gamma_1 = \frac{\sigma_{k+1}}{\sigma_k}$ and $\gamma_2 = \frac{\sqrt{\sum_{i=k+1}^{\min\set{n,d}} \sigma_i^2}}{\sigma_k}$. Let $\bX' \in (\cS_d)^m$ be a uniformly random $m$-size subset of the rows of $X$ (without replacement). Let $\Pi$ and $\bPi'$ be the projection matrices to the top-$k$ rows subspace of $X$ and $\bX'$, respectively.  
	Assuming that $\sigma_{k}^2 \geq 0.01 n/k$, then the following holds for $m = \tilde{\Theta}(k)$:
	\begin{enumerate}

		\item If $\gamma_1 \leq \frac{m}{2n}$, then $\pr{\norm{\Pi - \bPi'} \leq O\paren{\sqrt{\frac{n}{m}} \cdot \gamma_1}} \geq 0.9.\:$ ($\norm{\cdot}$ denotes the Spectral norm\footnote{The spectral norm of a matrix $A \in \bbR^{n\times d}$ is defined by $\norm{A} = \sup_{x \in \cS_d} \norm{A x}_2$ and is equal to $\sigma_1(A)$.}).\label{item:technique:S-diff}
		
		\item If $\gamma_2 \leq 0.1$, then $\pr{\norm{\Pi - \bPi'}_F \leq O(\gamma_2)} \geq 0.9.\:$\footnote{The Frobenius norm of a matrix $A \in \bbR^{n \times d}$ is equal to $\sqrt{\sigma_1(A)^2 + \ldots + \sigma_{\min\set{n,d}}(A)^2}$.}\label{item:technique:F-diff}
	\end{enumerate}
\end{lemma}

Namely, \cref{item:technique:S-diff} bounds the expected spectral norm distance of the projection matrices using the first type gap $\sigma_{k+1}/\sigma_{k}$ (which is used in the analysis of our strong estimator), and \cref{item:technique:F-diff} bounds the expected Frobenius norm distance using the second type gap $\sqrt{\sum_{i=k+1}^{\min\set{n,d}} \sigma_i^2}/\sigma_k$ (which is used in the analysis of our weak estimator). 
We prove \cref{lemma:technique:Pi-distance-bounds} in \cref{sec:key-property}.

The next step is to aggregate the non-private projection matrices $\Pi_1,\ldots,\Pi_t$ into a private one $\tPi$.
We consider two types of aggregations. The first one simply treats each matrix as a $d^2$ vector and privately estimate the average of $\Pi_1,\ldots,\Pi_t$.
The second type (which outperforms the first one in most cases) follows a similar high-level structure of \cite{SS21}. That is, to sample i.i.d.\ reference points $p_1,\ldots,p_q \sim \cN(\vec{0},\bbI_{d\times d})$ for $q = \Theta(k)$, privately average the $qd$-dimensional points $\set{(\Pi_j p_1,\ldots, \Pi_j p_q)}_{j=1}^t$ for obtaining a private $\tP \in \bbR^{q\times d}$ (whose \ith row estimates the projection of $p_i$ onto the top-$k$ rows subspace of $X$), and then compute the projection matrix of the top-$k$ rows subspace of $\tP$. But unlike \cite{SS21} who perform this step using a histogram-based averaging that has the same flavor of \cite{KV18}, we perform this step using FriendlyCore \cite{FriendlyCore22} that simplifies the construction and makes it practical in high dimensional regimes. We remark that in both aggregation types, we need a DP averaging algorithm that is resilient to a constant fraction of outliers (say, $20\%$) since both items in \cref{lemma:technique:Pi-distance-bounds} only guarantee that the expected number of outliers is no more than $10\%$. Fortunately, FriendlyCore can be utilized for such regimes of outliers (see \cref{sec:prelim:FriendlyCore} for more details). 

A few remarks are in order.

\begin{remark}
	The first aggregation type (which privately estimate the average of $\Pi_1,\ldots,\Pi_t$ directly) outperforms the second type only for our weak estimator in the regime $k \geq \sqrt{d}$ (as it inherently posses larger dependency in the dimension).
\end{remark}

\remove{
	\begin{remark}\label{remark:Z}
		The differences between the rates $n = n(\lambda)$ of weak and strong subspace estimators comes from the additional $\sqrt{\frac{n}{m}}$ factor in \cref{item:technique:S-diff} of \cref{lemma:technique:Pi-distance-bounds}. This difference arises from the following reason: If we have a matrix $Z \in \bbR^{n \times d}$ and we sample a random $m$-size subset of rows to obtain a matrix $\bZ' \in \bbR^{m \times d}$, then it can be easily shown that the Frobenius norm is expected to be reduced in the same rate, i.e., $\ex{\norm{\bZ'}_F^2} = \frac{m}{n} \norm{Z}_F^2$. Yet, this is not the case for spectral norm, and there are $Z$'s which have $\ex{\norm{\bZ'}} =  \norm{Z}$ (e.g., the identity matrix).  In our analysis, we write the SVD of $X$ in the form $X = \sum_{i=1}^{\min\set{n,d}} \sigma_i u_i v_i^T$ and use $Z =  \sum_{i=k+1}^{\min\set{n,d}} \sigma_i u_i v_i^T$. We pay the $\sqrt{\frac{n}{m}}$ factor in our spectral norm distance analysis because of this reason.
	\end{remark}
}

\begin{remark}\label{remark:large_sigma_k}
	We could avoid the requirement $\sigma_k^2 \geq 0.01 n/k$ by adding an additional parameter $\eta$ such that $\sigma_k^2 \geq \eta \cdot n/k$, and using subsets of size $m = \tilde{\Theta}(k/\eta)$ (which would increase $n$ by the same factor of $1/\eta$). For readability purposes, we chose to avoid this additional parameter. Because our algorithms provide useful projection by estimating the actual top-$k$ projection\remove{(see \cref{remark:accuracy})}, then such a requirement is unavoidable if we would like to provide utility guarantees only as a function of the singular values.\footnote{To illustrate why $\sigma_1,\ldots,\sigma_k$ should be large, consider a matrix $X$ whose first $n-k+1$ rows are $e_1$, and the next $k-1$ rows are $e_2,\ldots,e_k$. This matrix has $\sigma_2 = \ldots = \sigma_k = 1$, and even though it is a rank-$k$ matrix, it is clearly impossible to output a projection matrix that reveals any of the directions $e_2,\ldots,e_k$ under DP.} In fact, any assumption that would imply that two random subsets induce similar top-$k$ projection matrices would suffice for the utility analysis.% (e.g., the Gaussian assumption of \cite{SS21}).
\end{remark}

\begin{remark}\label{remark:unknown_gamma_k}
	
	To eliminate the known-$\gamma$ assumption, we can replace the FriendlyCore-averaging step \cite{FriendlyCore22} (that requires to know the diameter) with their \emph{unknown} diameter variant that gets two very rough bounds $\xi_{\min}$ and $\xi_{\max}$, and performs a private binary search for estimating a good diameter $\xi \in [\xi_{\min}, \xi_{\max}]$ in a preprocessing phase. This step only replaces the dependency on $d$ with $d + \log\log(\xi_{\max}/\xi_{\min})$ in the asymptotic sample complexity (section 5.1.2. in \cite{FriendlyCore22}) and is very practical. In fact, we use this method in our empirical evaluation in \cref{sec:experiments}.
	
	For handling unknown values of $k$, note that our algorithms provide useful utility guarantees (compared to the additive-gap based ones) only when $\sigma_{k+1}^2 \ll 1$. So in cases where $\sigma_k^2 \geq 2\cdot \log(k/\beta)/\eps'$ for some (fixed) privacy budget $\eps' < \eps$ (say, $\eps' = \eps/10$),  we can privately determine w.p. $1-\beta$ the right value of $k$ using a simple $\eps'$-DP method. The main observation is that the $\ell_1$ sensitivity of the vector $(\sigma_1^2,\ldots,\sigma_n^2)$ is at most $2$ (\cite{ADKMV19}), yielding that we can privately compute $(\sigma_1'^2, \ldots, \sigma_n'^2) = (\sigma_1^2 + \Lap(2/\eps'), \ldots, \sigma_n^2 + \Lap(2/\eps'))$, and then perform analysis on $(\sigma_1'^2, \ldots, \sigma_n'^2)$ to set $k$ as the first index $i$ where $\sigma_i^2 \geq \log(i/\beta)/\eps'$ and $\sigma_{i+1}^2 < \log(i/\beta)/\eps'$.
\end{remark}

\subsection{Lower Bounds}\label{sec:techniques:LB}

Our lower bounds (\cref{thm:intro:lower_bound:weak,thm:intro:lower_bound:strong}) use the recent framework of \cite{PTU24} for generating smooth lower bounds for DP algorithms using Fingerprinting Codes (FPC), but require technically involved analysis due to the complex structure of this problem for $k \geq 2$.

Roughly speaking, let $\cD$ be a distribution over $\oo^{n_0 \times d_0}$ that induces an optimal FPC codebook with $d_0 = \tilde{O}(n_0^2)$ (e.g., \cite{Tardos08,PTU24}). The connection between FPC and DP (first introduced by \cite{BUV14}) is that any DP algorithm, given a random codebook $X = (x_i^j)_{i\in[n_0], j \in [d_0]} \sim \cD$ as input, cannot output a vector $q=(q^1,\ldots,q^{d_0}) \in \oo^{d_0}$  that ``agrees" with most of the ``marked" columns of $X$ (Formally, for $b \in \oo$, a columns $x^j = (x_1^j,\ldots,x_n^j)$ is called $b$-marked if $x_1^j = \ldots =x_n^j = b$, and $q$ agrees with it if $q^j = b$).

Now consider a DP mechanism $\Mc \colon \cX^n \rightarrow \cW$ that satisfies some non-trivial accuracy guarantee.
\cite{PTU24} reduces the task of lower bounding $n$ to the following task: (1) Generate from an FPC codebook $X\in \oo^{n_0 \times d_0}$ hard instances $Y \in \cX^n$ for $\Mc$, and (2) Extract from the output $w \sim \Mc(Y)$ a vector $q \in \oo^{d_0}$ that agrees with most of the marked columns of  $X$ ($n_0$ and $d_0$ are some functions of $n$, $\cX$ and the weak accuracy guarantee of $\Mc$). \cite{PTU24} proved that if there exists such generating algorithm $\Gc$ and extracting algorithm $\Fc$ (which even share a random secret that $\Mc$ does not see) such that $\Gc$ is \emph{neighboring-preserving} (i.e., maps neighboring databases to neighboring databases), then it must hold that $n_0 \geq \tilde{\Omega}(\sqrt{d_0})$ (Otherwise, $\Mc$ cannot be DP).

\paragraph{Warm-up: DP averaging.} We first sketch how \cite{PTU24} applied their framework with $n_0 = n$ and $d_0 = \Theta(d/\lambda^2)$ for proving a lower bound for the simpler problem of DP averaging. In this setting, we are given a mechanism that guarantees $\lambda \gamma$-accuracy ($\ell_2$ additive error) for $\gamma$-easy instances (i.e., points that are $\gamma$-close to each other in $\ell_2$ norm). The generator $\Gc$, given an FPC codebook $X\in \oo^{n_0 \times d_0}$, uses the \emph{padding-and-permuting} technique: It pads $\ell \approx 10^4 \lambda^2 d_0$ $1$-marked columns and $\ell$ $(-1)$-marked columns, and then permutes all the $d = d_0 + 2\ell$ columns of the new codebook $X'$ using a random permutation $\pi\colon [d] \rightarrow [d]$ that is shared with the extractor $\Fc$. The input $Y$ to the algorithm would be the \emph{normalized} rows of $X'$ which are $\frac1{100 \lambda}$-close to each other in $\ell_2$ norm, so the mechanism has to output an $\frac1{100}$-accurate solution $w$. In particular, after rounding $w$ to $\oo^d$, the coordinates of $w$ must agree with a vast majority of the marked columns, and also with a vast majority of the original marked columns that are located within $\pi(1),\ldots,\pi(d_0)$ as it cannot distinguish between them and the other marked columns (because $\pi$ is hidden from it). The extractor $\Fc$, given $w$ and $\pi$, rounds $w$ to $\oo^d$ and outputs $q = (w^{\pi(1)}, \ldots, w^{\pi(d)})$ which agrees with most of the marked columns of $X$. Hence, we obtain the lower bound of $n \geq \tilde{\Omega}(\sqrt{d_0}) = \tilde{\Omega}(\sqrt{d}/\lambda)$.

\paragraph{DP Subspace Estimation} In our case, we are given a (weak or strong) subspace estimator $\Mc \colon (\cS_d)^n \rightarrow \bbR^{d\times d}$ that outputs an $\lambda \gamma$-useful rank-$k$ projection matrix if $\sigma_{k+1}(X) \leq \gamma \sigma_{k}(X)$ (or $\sqrt{\sum_{i=k+1}^{\min\set{n,d}} \sigma_i(X)^2} \leq \gamma \sigma_{k}(X)$). We prove our lower bounds by applying the framework with $n_0 = n/k$ and $d_0 = \Theta(\alpha d)$, for some parameter $\alpha = \alpha(\lambda)$ that will depend on the type of $\Mc$ we consider.
In order to generate hard instances $Y \in (\cS_d)^n$ for $\Mc$ given an FPC codebook $X \sim \cD\:\:$ ($X \in \oo^{(n/k)\times d_0}$), we use a variation of the approach to \cite{DTTZ14}. Namely, our generator $\Gc$ samples $k$ independent FPC codebooks $A_1,\ldots,A_k \sim \cD$ where it plants $A_i= X$ for a random $i \la [k]$. Then for each $j \in [k]$, it applies (independently) the padding-and-permuting technique of \cite{PTU24} where it pads $\ell$ $1$-marked columns and $\ell$ $(-1)$-marked columns for $\ell \approx \frac{d_0}{2 \alpha}$, and permute all the columns. This induces $k$ matrices $B_1, \ldots,B_k \in \oo^{(n/k) \times d}$ (for $d = d_0 + 2\ell$) such that each $B_j$ is ``almost''  rank-$1$ and their vertical concatenation $B\in \oo^{n \times d}$ is almost rank-$k$. It provides $Y = \frac1{\sqrt{d}} B$ as the input for $\Mc$. See \cref{hard-instances} for graphical illustrations.

\begin{figure*}
	%\centerline{\includegraphics[scale=.3]{avg-R.png}
		\centerline{
			\includegraphics[scale=.40]{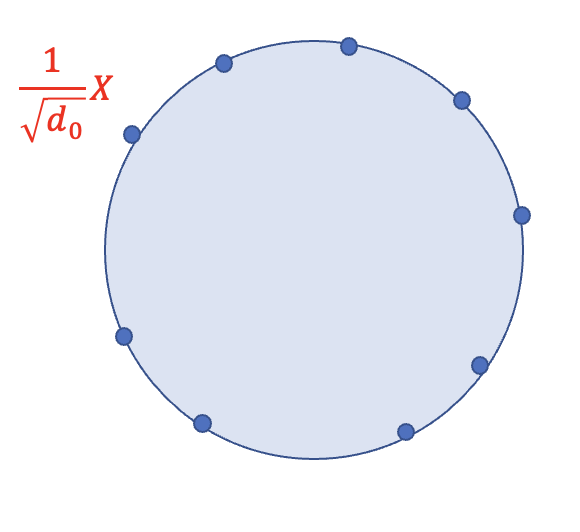}
			\includegraphics[scale=.40]{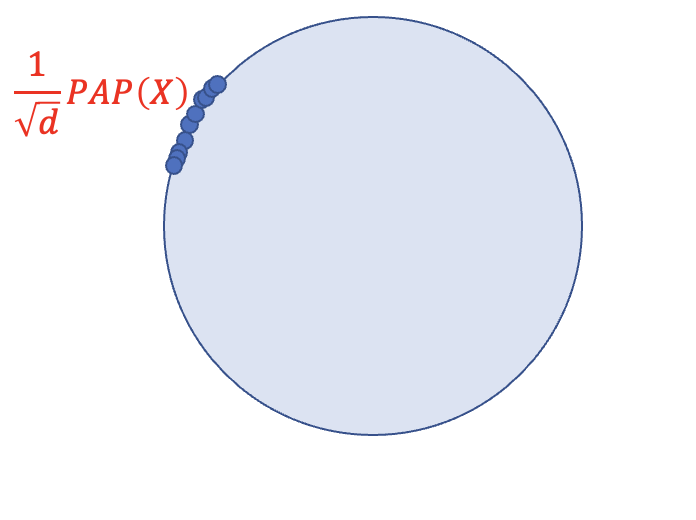}
			\includegraphics[scale=.40]{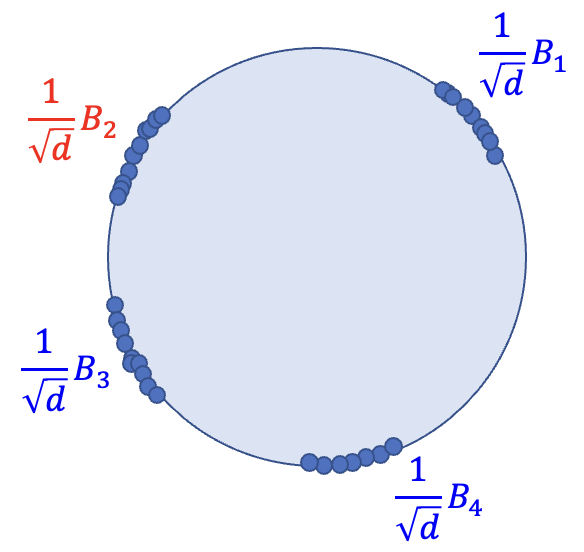}		
		}
		\caption{From Left to Right: (1) The normalized rows of the fingerprinting codebook $X$ are well-spread on the $d_0$-dimensional unit sphere. (2) Applying the padding-and-permuting (PAP) technique makes the normalized points very close to each other on the $d$-dimensional unit sphere ($d \gg d_0$). (3) We create hard instances for DP subspace estimation using $k$-independent (normalized) PAP-FPC codebooks $B_1,\ldots,B_k$, where $PAP(X)$ is planted in one of the $B_i$'s (in this example, in $B_2$). Reducing $\alpha$ (i.e., increasing the padding length) makes the points in each group $B_i$ closer to each other, which in particular, increases the closeness to a $k$-dimensional subspace.}
		\label{hard-instances}
	\end{figure*}

We remark that at this step, the main difference from \cite{DTTZ14} (who implicitly follow a similar paradigm) is that they use a fixed padding length of $\ell = 15 d_0$ that suffices for the robustness properties that they need. On the other hand, we use \cite{PTU24}'s observation that increasing the padding can handle low-accuracy regimes of many problems, and indeed we use the padding length $\ell$ to increase the $k$'th singular value gap, which will be a function of the quality parameter $\lambda$.%Another difference is that \cite{DTTZ14} only use padding of $1$-marked columns, and therefore they needed to apply an additional rotation step to each $B_i$ in order to get almost orthogonal directions. In this work, because we pad equal number of $1$-marked and $(-1)$-marked columns and randomly permute them (using a different permutation in each part), then it already induces almost orthogonal directions without the need of the rotation step, which gives the advantage that our hard instances are boolean (up to the normalization factor of $1/\sqrt{d}$). 

The next step is to choose the right value of $\alpha = \alpha(\lambda)$ such that $\Mc$, on input $Y$, will have to output a useful projection matrix.
We show that the input matrix $Y$ has w.h.p. $\sigma_1(Y)^2 \geq \ldots \geq \sigma_k(Y)^2 \geq (1 - O(\alpha)) \cdot \frac{n}{k}$, which yields that $\sum_{i=k+1}^{\min\set{n,d}} \sigma_i(Y)^2 \leq O(\alpha) n$.
If $\Mc$ is a weak estimator, then we simply use $\alpha = \Theta\paren{\frac{1}{\lambda^2 k}}$ to guarantee that $\sum_{i=k+1}^{\min\set{n,d}} \sigma_i(Y)^2 \leq \gamma^2\cdot \sigma_k(Y)^2$ for $\gamma = \frac1{1000 \lambda}$, which yields that by the utility guarantee of $\Mc$, we get an $0.001$-useful projection matrix.
If $\Mc$ is a strong estimator, then we use $\alpha = \Theta\paren{\frac{1}{\lambda^2}}$ (i.e., we decrease the padding length by a factor of $k$). Yet, in order to meet the requirements of $\Mc$, we must argue that w.h.p., $\sigma_{k+1}(Y)^2 \leq \tilde{O}(\alpha) \cdot \frac{n}{k}$, and this is more complex than the previous case. Here we use more internal properties of the fingerprinting distribution $\cD$. Namely, that in \cite{PTU24}'s construction (which is also true for \cite{Tardos08}'s one), each entry of the codebook matrix has expectation $0$ and the columns of the matrix are independent. Using known concentration bounds, this allows us to argue that if we pick a unit vector $v$ that is orthogonal to the top-$k$ rows subspace of $Y$, then with probability at least $1 - \exp(-\tilde{\Omega}(d))$ it holds that $\norm{Y v}_2^2 \leq \tilde{O}(\alpha)\cdot \frac{n}{k}$. Since $\sigma_{k}^2$ is bounded by the supremum of $\norm{Y v}_2^2$ under such unit vectors, we conclude the proof of this part using a net argument.

Finally, the last step, which is not trivial for $k \geq 2$, is to extract from an $0.001$-useful projection matrix $\tPi$ for $Y$, a vector $q \in \oo^{d_0}$ that with noticeable probability, strongly agrees with the marked columns of the original codebook $X \in \oo^{n_0 \times d_0}$. For that, our extractor $\Fc$ uses the random permutations and the random location $i$ (which are part of the shared secret between the generator and the extractor) and follows the strategy of  \cite{DTTZ14}. That is, it applies the \ith invert permutation over the columns of $\tPi$ (denote the resulting matrix by $\tPi_i$), chooses a vector $u \in Span(\tPi_i)$ that has the maximal agreement with \emph{half} of the padding bits, and then simply outputs its first $d_0$ coordinates after rounding to $\oo$. The intuition is that an $0.001$-useful projection matrix must be also $0.001$-useful for at least one of the parts $Y_j = \frac1{\sqrt{d}}B_j$. Since all the $Y_j$'s have the same distribution and the location $i$ (where the original $X$ is planted) is hidden from $\Mc$, then it must be $0.001$-useful for $Y_i$ with probability at least $\beta/k$ (where $\beta$ denotes the success probability of $\Mc$). Given that this event occurs, the usefulness of $\tPi$ implies that there must exists a vector in $Span(\tPi_i)$ that strongly agrees with half of the padding locations. But because all the marked columns (that includes the padding locations) are indistinguishable from the eyes of $\Mc$ who computed $\tPi$, then a similar agreement must hold for the marked columns of $X$.  

\section{Preliminaries}\label{sec:prelim}

\subsection{Notations}
We use calligraphic letters to denote sets and distributions, uppercase for matrices and datasets, boldface for random variables, and lowercase for vectors, values and functions. 
For $n \in \bbN$, let $[n] = \set{1,2,\ldots,n}$. 
Throughout this paper, we use $i \in [n]$ as a row index, and $j \in [d]$ as a column index (unless otherwise mentioned).

For a matrix $X = (x_i^j)_{i\in [n], j \in [d]}$, we denote by $x_i$ the \ith row of $X$ and by $x^j$ the \jth column of $X$. A column vector $x \in \bbR^n$ is written as $(x_1,\ldots,x_n)$ or $x = x_{1\ldots n}$, and a row vector $y \in \bbR^d$ is written as $(y^1,\ldots,y^d)$ or $y^{1\ldots d}$. In this work we consider mechanisms who receive an
$n \times d$ matrix $X$ as input, which is treated as the dataset $X = (x_1,\ldots,x_n)$ where the rows of $X$ are the elements (and therefore, we sometimes write $X \in (\bbR^d)^n$ instead of $X \in \bbR^{n \times d}$ to emphasize it).

For a vector $x \in \bbR^d$ we define $\norm{x}_2 = \sqrt{\sum_{i=1}^n x_i^2}$ (the $\ell_2$ norm of $x$), and for a subset $\cS \subseteq [d]$ we define $x_{\cS}=(x_i)_{i \in \cS}$, and in case $x$ is a row vector we write $x^{\cS}$. Given two vectors $x = (x_1,\ldots,x_n), y = (y_1,\ldots,y_n)$, we define $\ip{x,y} = \sum_{i=1}^n x_i y_i$ (the inner-product of $x$ and $y$). We denote by $\cS_d$ the set of $d$-dimensional unit vectors, that is, $\cS_d = \set{v \in \bbR^d \colon \norm{v}_2 = 1}$.
For a matrix $X = (x_i^j)_{i \in [n], j \in [d]}$ we let $\norm{X} = \max_{v \in \cS_d} \norm{X v}_2$ (the \emph{spectral norm} of $X$) and let $\norm{X} _F= \sqrt{\sum_{i \in [n], j \in [d]} (x_i^j)^2}$ (the Forbenius norm of $X$).
For a matrix $X = (x_i^j)_{i\in [n], j \in [d]} \in \oo^{n \times d}$ and $b \in \oo$, we define the $b$-marked columns of $X$ as the subset $\cJ^b_X \subseteq [d]$ defined by $\cJ^b_X = \set{j \in [d] \colon x_i^j = b \text{ for all }i\in[n]}$.

For $d \in \bbN$ we denote by $\cP_d$ the set of all $d \times d$ permutation matrices. For a permutation matrix $P \in \cP_d$ and $i \in [d]$ we denote by $P(i)$ the index $j \in [d]$ for which $e_i P= e_j$ (where $e_i$ and $e_j$ are the corresponding one-hot row vectors), and for $\cI \subseteq [d]$ we denote $P(\cI) = \set{P(i) \colon i\in \cI}$. 

For $d,k \in \bbN$ we denote by $\cW_{d,k}$ the set of all $d \times d$ projection matrices of rank $k$. For a matrix $A \in \bbR^{n \times d}$ we denote by $Span(A)$ the columns subspace of $A$ (and therefore the rows subspace of $A$ is $Span(A^T)$).

For $z \in \bbR$,  we define $\signn{z} \eqdef \begin{cases} 1 & z \geq 0 \\ -1 & z<0\end{cases}$ and for $v = (v^1,\ldots,v^d) \in \bbR^d$ we define $\signn{v} \eqdef (\signn{v^1},\ldots,\signn{v^d}) \in \oo^d$.

\subsection{Distributions and Random Variables}\label{sec:prelim:dist}

Given a distribution $\cD$, we write $x \sim \cD$ to denote that $x$ is sampled according to $\cD$.
For a set $\cS$, we write $x \la \cS$ to denote that $x$ is sampled from the uniform distribution over $\cS$.

\subsection{Singular Value Decomposition (SVD)}

For a matrix $X \in \bbR^{n \times d}$, the singular value decomposition of $X$ is defined by $X = U \Sigma V^T$, where $U \in \bbR^{n \times n}$ and  $V \in \bbR^{d \times d}$ are unitary matrices. The matrix $\Sigma \in \bbR^{n \times d}$ is a diagonal matrix with non-negative entries $\sigma_1 \geq \ldots \geq \sigma_{\min\set{n,d}} \geq 0$ along the diagonal, called the singular values of $X$. The SVD of $X$ can also be written in the form $X = \sum_{i=1}^{\min\set{n,d}} \sigma_i u_i v_i^T$, where $u_i$ and $v_i$ are the \ith columns of $U$ and $V$ (respectively). It holds that $\norm{X}_F^2 = \sum_{i} \sigma_i^2$ and $\norm{X} = \sigma_1$. We define the top-$k$ rows subspace of $X$ as the subspace spawned by the first $k$ columns of $V$.

%The top (right) singular vector of $X$ is defined by the first column of $V$ (call it $v_1 \in \cS_d$) which satisfy $\norm{X \cdot v_1}_2 = \max_{v \in \cS_d} \norm{X \cdot v}_2$.

\begin{fact}[Min-Max principle for singular values]\label{fact:min-max}
	For every matrix $X  \in \bbR^{n \times d}$ and $i \in [\min\set{n,d}]$ it holds that
	\begin{align*}
		\sigma_i(X) = \max_{\dim(\cE) = i} \: \min_{v \in \cS_d \cap \cE} \norm{X v}_2 = \min_{\dim(\cE) = i} \: \max_{v \in \cS_d \cap \cE} \norm{X v}_2.
	\end{align*}
\end{fact}

\subsection{Concentration Bounds}

%\begin{fact}[Hoeffding's inequality \cite{hoeff}]
%	%Hoeffding's inequality. 
%	Let $\bx_1,\dots, \bx_n$ be independent random variables taking integer values in the range $[a,b]$. 
%	Also, let $\bx=\sum_{i=1}^n \bx_i$ denote the sum of the variables and $\mu = \eex{}{\bx}$ denote its expectation.
%	Then, for any $t>0$,
%	\begin{align}
%		\pr{\size{\bx -\mu} \leq t} \leq 2\cdot e^{-\frac{2t^2}{n(b-a)^2}}.\label{hoeffding_ineq1}
%	\end{align}
%\end{fact}

\begin{fact}[\cite{MonSmith90}]\label{fact:ip-concentration}
	Let $\bx = (\bx_1,\ldots,\bx_n)$ where the $\bx_i$'s are i.i.d. random variables with $\pr{\bx_i=1} = \pr{\bx_i = -1} = 1/2$, and let $v \in \bbR^n$.
	Then 
	\begin{align*}
		\forall t \geq 0: \quad \pr{\size{\ip{\bx,v}} > t\cdot \norm{v}_2} \leq 2 \exp(-t^2/2).
	\end{align*}
\end{fact}

%\begin{fact}[Lemma 2.10 in \cite{SS21}]
%	Let $\bx_1,\ldots,\bx_q \sim \cN(\vec{0}, \Sigma)$ be random vectors in $\bbR^d$, where $\Sigma$ is the projection of $\bbI_{d \times d}$ onto a subspace of $\bbR^d$ of rank $k$. Then for all $t \geq 0$:
%	\begin{align*}
%		\pr{\forall i \in [q]: \: \norm{\bx_i}^2 \leq k + 2\sqrt{kt} + 2t} \geq 1 - q e^{-t}.
%	\end{align*}
%\end{fact}

\begin{fact}[Bernstein's Inequality for sampling without replacement (\cite{BardenetM15}, Proposition 1.4)]\label{fact:Bernstein}
	Let $\bx_1,\ldots,\bx_m$ be a random sample drawn without replacement from $\set{w_1,\ldots,w_n}$ for $n \geq m$.
	Let $a = \min_{i \in [n]} w_i$, $b = \max_{i \in [n]} w_i$, $\mu = \frac1n \sum_{i=1}^n w_i$ and $\sigma^2 = \frac1{n} \sum_{i=1}^n (w_i-\mu)^2$.
	Then for every $t\geq 0$,
	\begin{align*}
		\pr{\size{\sum_{i=1}^m \bx_i -m\cdot  \mu} \geq t} \leq 2\cdot \exp\paren{-\frac{t^2}{m \sigma^2 + (b-a) t/3}}.
	\end{align*}
\end{fact}

The following proposition is used for proving \cref{lemma:technique:Pi-distance-bounds}.

\begin{proposition}\label{prop:sigma-concent}
	Let $X = (x_1, \ldots,x_n) \in (\cS_d)^n$ with singular values $\sigma_1 \geq \ldots \geq \sigma_{\min\set{n,d}} \geq 0$.
	Let $\bX' = (\bx'_1,\ldots,\bx'_m)$  be the random matrix that is generated by taking a uniformly random $m$-size subset of the rows of $X$.
	Let $v \in \cS_d$ and $a^2 = \norm{X v}_2^2 = \sum_{i=1}^n \ip{x_i ,v}^2$. Assuming that $m \geq \frac{2 n \ln(1/\beta)}{9 a^2}$, it holds that  
	\begin{align*}
		\pr{\size{\norm{\bX' v}_2^2  - \frac{m}{n} a^2} \geq \sqrt{2 \frac{m}{n} a^2 \ln(1/\beta)}} \leq 2\beta.
	\end{align*}
\end{proposition}
\begin{proof}
	The random variable $\norm{\bX' v}_2^2$ is distributed as the sum of a uniformly random $m$-size subset of $(\ip{x_1,v}^2,\ldots,\ip{x_n,v}^2)$ (each element is bounded in $[0,1]$). In the notations of \cref{fact:Bernstein}, it holds that $\mu = \frac1n \sum_{i=1}^n \ip{x_i,v}^2 = a^2/n$, and 
	\begin{align*}
		\sigma^2 
		&= \frac1n \sum_{i=1}^n (\ip{x_i,v}^2 - a^2/n)^2\\
		&= \frac1n \sum_{i=1}^n (\ip{x_i,v}^4 - 2 \ip{x_i,v}^2 a^2/n + a^4/n^2)\\
		&= \frac1n \sum_{i=1}^n \ip{x_i,v}^4 - a^4/n^2\\
		&\leq \frac1n  \sum_{i=1}^n \ip{x_i,v}^2 = a^2/n,
	\end{align*}
	where the last inequality holds since $\ip{x_i,v}^2 \leq 1$.
	
	Let $t =  \sqrt{2 \frac{m}{n} a^2 \ln(1/\beta)}$. By \cref{fact:Bernstein},
	\begin{align*}
		\pr{\size{\norm{\bX' v}_2^2 - \frac{m}{n} a^2} \geq t} 
		&\leq 2\cdot \exp\paren{-\frac{t^2}{\frac{m}{n} a^2 + t/3}}\\
		&= 2\cdot  \exp\paren{-\frac{2 \frac{m}{n} a^2 \ln(1/\beta)}{\frac{m}{n} a^2 +  \sqrt{2 \frac{m}{n} a^2 \ln(1/\beta)}/3}}\\
		&\leq 2\cdot  \exp\paren{-\frac{2 \frac{m}{n} a^2 \ln(1/\beta)}{2 \frac{m}{n} a^2}}\\
		&\leq 2 \beta,
	\end{align*}
	where the penultimate inequality holds by the assumption on $m$.
\end{proof}

\subsubsection{Hypergeometric Distributions}

\begin{definition}\label{def:Hyper}
	For $n \in \N$, $m \in [n]$ and $w\in \set{-n,\dots,n}$, define the \emph{Hypergeometric} probability distribution $\HG_{n,m,w}$ 
	as the output of the following process: Take a vector $v \in \oo^n$ with $\sum_{i=1}^n v_i = w$, choose a uniformly random subset $\cI \subseteq [n]$ of size $m$, and output $\sum_{i \in \cI} v_i$.
\end{definition}

\begin{fact}[\cite{scala2009hypergeometric}, Equations 10 and 14]\label{fact:hyperHoeffding}
	If $\bx \sim \HG_{n,m,w}$ then 
	\begin{align*}
		\forall t\geq 0: \quad \pr{{\abs{\bx-\mu}} \geq t} \leq e^{-\frac{t^2}{2\ell}},
	\end{align*}
	where $\mu = \ex{\bx} = \frac{m  \cdot w}{n}$.
\end{fact}

%\remove{
%\subsection{Sub-Gaussian Distributions}
%
%\Enote{Maybe it is not needed.}
%
%\begin{definition}[Sub-Gaussian Random Variable and Norm]
%	A random variable $\bx$ over $\bbR$ is called $\lambda$-\emph{sub-gaussian} if for every $t\geq 0$: $\pr{\size{\bx} \geq t} \leq 2\exp(-t^2/\lambda)$. We say that $\bx$ is \emph{sub-Gaussian} if it is $\lambda$-\emph{sub-gaussian} for some $\lambda > 0$. The \emph{sub-gaussian norm} of $\bx$ is defined by
%	\begin{align*}
%		\norm{\bx}_{\psi_2} = \inf\set{t > 0 \colon \ex{\exp(\bx^2/t^2) \leq 2}}.
%	\end{align*}
%\end{definition}
%
%
%\begin{definition}[Sub-Gaussian Random Vector (\cite{Vershynin_2018})]
%	A random vector $\bx$ in $\bbR^d$ is called \emph{sub-gaussian} if the one-dimensional marginals $\ip{\px, v}$ are sub-gaussian random variables for all $v \in \bbR^d$. The \emph{sub-gaussian norm} of $\bx$ is defined as
%	\begin{align*}
%		\norm{\bx}_{\psi_2} = \sup_{v \in \cS_d} \norm{\ip{\bx, v}}_{\psi_2}.
%	\end{align*}
%\end{definition}
%}

\subsubsection{Sub-Exponential Distributions}

\begin{definition}[Sub-Exponential Random Variable and Norm]\label{def:sub-exp}
	We say that a random variable $\bx \in \bbR$ is \emph{sub-exponential} if there exists $t > 0$ such that $ \ex{e^{\size{\bx}/t}} \leq 2$. 
	The \emph{sub-exponential} norm of $\bx$, denoted $\norm{\bx}_{\psi_1}$, is
	\begin{align*}
		\norm{\bx}_{\psi_1} = \inf\set{t > 0 \colon \ex{e^{\size{\bx}/t}} \leq 2}.
	\end{align*} 
\end{definition}

\begin{fact}[Bernstein's inequality (Theorem 2.8.1 in \cite{Vershynin_2018})]\label{fact:sub-exp-concent}
	Let $\bx_1,\ldots,\bx_n$ be independent, mean zero, sub-exponential random variables. Then
	\begin{align*}
		\forall t \geq 0: \quad \pr{\size{\sum_{i=1}^n \bx_i} \geq t} \leq 2\exp\paren{-\Omega\paren{\min\paren{\frac{t^2}{\sum_{i=1}^n \norm{\bx_i}_{\psi_1}^2}, \: \frac{t}{\max_i \norm{\bx_i}_{\psi_1}}}}}.
	\end{align*}
\end{fact}

\subsection{Nets}

\begin{definition}[$\gamma$-Net]\label{def:net}
	Let $\cT$ be a subspace of $\bbR^d$. Consider a subset $\cK \subset T$ and let $\gamma > 0$. A subset $\cN \subseteq \cK$ is called an \emph{$\gamma$-net} of $\cK$ if  every point in $\cK$ is within distance $\gamma$ of some point of $\cN$, i.e., 
	\begin{align*}
		\forall x \in \cK \: \exists y \in \cN \: : \: \norm{x - y}_2 \leq \gamma
	\end{align*}
\end{definition}

\begin{fact}[Extension of Corollary 4.2.13 in \cite{Vershynin_2018}]\label{fact:net}
	If $E$ is a subspace with $\dim(E) = k$, then there exists an $\gamma$-net of size $\paren{3/\gamma}^k$ to the unit sphere in $E$ (i.e., to $E \cap \cS_d$).
\end{fact}

\subsection{Projections}

Recall that $\cW_{d,k}$ denotes the set of all $d \times d$ rank-$k$ projection matrices.
For a matrix $A \in \bbR^d$, we denote that $\Pi_A$ the projection matrix onto the subspace spawned by the columns of $A$ (in case $A$ is a unitary matrix, $\Pi_A = AA^T$).

\begin{fact}[Theorem 1 and Lemma 1 in \cite{CZ16}]\label{fact:pert_bound}
	Let $X, Y, Z \in \bbR^{n \times d}$ such that $X = Y + Z$.
	Let $[U \: U_{\perp}] \Sigma [V \: V_{\perp}]^T$ be the SVD of $Y$, and let $[\hU \: \hU_{\perp}] \hSigma [\hV \: \hV_{\perp}]^T$ be the SVD of $X$, 
	where $U,V,\hU,\hV$ denote the first $k$ columns of $[U \: U_{\perp}], [V \: V_{\perp}], [\hU \: \hU_{\perp}],  [\hV \: \hV_{\perp}]$ (respectively).
	Let $Z_{12} = \Pi_{U} Z \Pi_{V_{\perp}}$ and $Z_{21} = \Pi_{U_{\perp}} Z \Pi_V$.  In addition, let $z_{ij} = \norm{Z_{ij}}$, let $\alpha = \sigma_{\min}(U^T X V)$ (i.e., the smallest singular value larger than $0$), and $\beta = \norm{U_{\perp}^T X V_{\perp}}$. If $\alpha^2 > \beta^2 + \min\set{z_{12}^2,z_{21}^2}$, then
	\begin{align*}
		\norm{\Pi_{V} - \Pi_{\hV}} \leq 2 \cdot  \frac{\alpha z_{12} + \beta z_{21}}{\alpha^2 - \beta^2 - \min\set{z_{12}^2,z_{21}^2}}
	\end{align*}
	and
	\begin{align*}
		\norm{\Pi_{V} - \Pi_{\hV}}_F \leq \sqrt{2}\cdot  \frac{\alpha \norm{Z_{12}}_F + \beta \norm{Z_{21}}_F}{\alpha^2 - \beta^2 - \min\set{z_{12}^2,z_{21}^2}}.
	\end{align*}
\end{fact}

\begin{proposition}\label{cor:proj-diff}
	Let $X,Y, Z, V, \hV$ as in \cref{fact:pert_bound} such that $Y$ has rank-$k$ and $Span(Y^T)$ is orthogonal to $Span(Z^T)$. If $\sigma_k(Y)^2 \geq 2 \norm{Z}^2$, then
	\begin{enumerate}
		\item $\norm{\Pi_{V} - \Pi_{\hV}} \leq 4 \cdot \frac{ \norm{Z}}{\sigma_k(Y)}$, and\label{item:proj-diff:S-norm}
		\item $\norm{\Pi_{V} - \Pi_{\hV}}_F \leq 2\sqrt{2} \cdot \frac{ \norm{Z}_F}{\sigma_k(Y)}$.\label{item:proj-diff:F-norm}
	\end{enumerate}
\end{proposition}
\begin{proof}
	Note that $Span(V) = Span(Y^T)$.
	Therefore $Z V = 0$, which implies that $Z_{21} = 	\Pi_{U_{\perp}} Z \Pi_V = 0$.
	Compute
	\begin{align*}
		\alpha 
		&= \sigma_{\min}(U^T X V) \\
		&= \sigma_{\min}(U^T Y V + U^T \underbrace{Z V}_0)\\
		&= \sigma_k(Y).
	\end{align*}
	\begin{align*}
		\beta
		&= \norm{U_{\perp}^T X V_{\perp}}\\
		&= \norm{U_{\perp}^T \underbrace{Y V_{\perp}}_0 + U_{\perp}^T Z V_{\perp}}\\
		&\leq \norm{U_{\perp}} \cdot \norm{Z} \cdot \norm{V_{\perp}}\\
		&\leq \norm{Z}.
	\end{align*}
	\begin{align*}
		\norm{Z_{12}} = \norm{\Pi_U \cdot Z \cdot \Pi_{V_{\perp}}} \leq \norm{Z},
	\end{align*}
	\begin{align*}
		\norm{Z_{12}}_F \leq  \norm{\Pi_U \cdot Z \cdot \Pi_{V_{\perp}}}_F \leq \norm{Z}_F,
	\end{align*}
	where the last inequalities in the above two equations hold since for any matrix $A$ and projection matrices $\Pi_1, \Pi_2$ it holds that $\norm{\Pi_1 A \Pi_2} \leq \norm{A}$ and $\norm{\Pi_1 A \Pi_2}_F \leq \norm{A}_F$.
	The proof now immediately follows by applying \cref{fact:pert_bound}.
\end{proof}

\begin{proposition}\label{prop:P-Pprime}
	Let $P \in \bbR^{n \times d}$ be a rank-$k$ matrix and let $P'\in \bbR^{n \times d}$ such that $\norm{P-P'}_{F} \leq \alpha$.
	Let $\Pi$ be the projection matrix onto $Span(P^T)$, and let $\Pi'$ be the projection onto the top-$k$ rows subspace of $P'$.
	If $\sigma_k(P) \geq 2\alpha$, then
	\begin{align*}
		\norm{\Pi - \Pi'}_F \leq 2\sqrt{2} \cdot \frac{\alpha}{\sigma_k(P) - \alpha}.
	\end{align*}
\end{proposition}
\begin{proof}
	Let $E = P' - P$, and divide $E$ into $E = E_{P} + E_{\bar{P}}$ where the rows of $E_P$ belong to the rows subspace of $P$ and the rows of $E_{\bar{P}}$ are orthogonal to it.
	Let $Y = P + E_{P}$, so we can write $P' = Y + E_{\bar{P}}$. Note that $\norm{E_{\bar{P}}}_F \leq \alpha$, and $\sigma_k(Y) \geq \sigma_k(P) - \norm{E_{P}} \geq \sigma_k(P) - \alpha$.
	The proof now follows by applying \cref{cor:proj-diff}(\ref{item:proj-diff:F-norm}) on $P', Y, E_{\bar{P}}$.
\end{proof}

\begin{fact}[Implied by Corollary 4.6 in \cite{SS21}]\label{fact:SS-Pips}
	Let $\Pi, \Pi_1,\ldots,\Pi_t \in \cW_{d,k}$ s.t. for all $j \in [t]$, $\norm{\Pi - \Pi_j} \leq \alpha$.   
	Let $\bp_1,\ldots,\bp_q$ be i.i.d. random vectors in $\bbR^d$ from $\cN(\vec{0}, \bbI_{d\times d})$. Then 
	\begin{align*}
		\pr{\forall i\in [q], j \in [t], \: \norm{\paren{\Pi - \Pi_j}p_i}_2 \leq O\paren{\alpha\paren{\sqrt{k} + \sqrt{\ln(qt)}}}} \geq 0.95.
	\end{align*}
\end{fact}

\begin{fact}[Implied by the proof of Lemma 4.9 in  \cite{SS21}]\label{fact:SS-pi-span}
	Let $\Pi \in \cW_{d,k}$ and let $\bp_1,\ldots,\bp_q$ be i.i.d. random vectors in $\bbR^d$ from $\cN(\vec{0}, I_{d\times d})$. Let $\bP$ be the $d \times q$ matrix whose columns are $\Pi \bp_1,\ldots,\Pi \bp_q$. If $q \geq c\cdot k$ for some large enough constant $c$, then w.p. $0.95$ it hold that $\sigma_k(\bP) \geq \Omega(\sqrt{k})$ (and in particular, $Span(\bP) = Span(\Pi)$).
\end{fact}

\begin{proposition}\label{prop:F-dist-implies-estimator}
	For any $\Pi, \tPi \in \cW_{d,k}$ and $X \in (\cS_d)^n$, it holds that
	\begin{align*}
		\norm{\Pi \cdot X^T}_F^2 - \norm{\tPi \cdot X^T}_F^2 \leq 2n \cdot \norm{\Pi - \tPi}_F.
	\end{align*}
\end{proposition}
\begin{proof}
	Compute
	\begin{align*}
		\norm{\Pi X^T}_F^2 - \norm{\tPi X^T}_F^2
		&= (\norm{\Pi X^T}_F - \norm{\tPi X^T}_F) \cdot (\norm{\Pi X^T}_F + \norm{\tPi X^T}_F)\\
		&\leq (\norm{\Pi X^T}_F - \norm{\tPi X^T}_F) \cdot 2\sqrt{n}\\
		&\leq \norm{(\Pi-\tPi)X^T}_F \cdot 2\sqrt{n}\\
		&\leq \norm{\Pi-\tPi}_F\cdot \norm{X^T}_F \cdot 2\sqrt{n}\\
		&= 2n \cdot \norm{\Pi - \tPi}_F.
	\end{align*}
\end{proof}

\begin{proposition}\label{prop:gram-schmidt}
	Let $v_1,\ldots,v_k \in \cS_d$ with $\max_{i,j} \size{\ip{v_i,v_j}} \leq \alpha \leq \frac1{20}$.
	Let $u_1,\ldots,u_k$ be the result of the Gram-Schmidt process applied on  $v_1,\ldots,v_k$.
	Then for every $i \in [k]$, there exists $\lambda_{i-1} \in \bbR$ with $\size{\lambda_{i-1}} \leq \alpha(1 + 4\alpha)$ and $w_i \in \bbR^d$ with $\norm{w_i}_2 \leq 2\alpha^2$ such that
	\begin{align*}
		u_i = v_i + \lambda_{i-1} v_{i-1} + w_i.
	\end{align*}
\end{proposition}
\begin{proof}
	We prove it by induction on $i$.
	The case $i=1$ holds since $u_1 = v_1$.
	Assume it holds for $i$, and we prove it for $i+1$. 
	Define $\lambda_i = -\ip{u_{i},v_{i+1}}$ and $w_{i+1} = \lambda_i(\lambda_{i-1} v_{i-1} + w_i)$.
	Note that
	\begin{align*}
		\size{\lambda_i}
		&= \size{\ip{v_i + \lambda_{i-1} v_{i-1} + w,\: v_{i+1}}}\\
		&= \size{\ip{v_i, v_{i+1}} + \lambda_{i-1} \ip{v_{i-1}, v_{i+1}} + \ip{w, v_{i+1}}}\\
		&\leq \alpha + \size{\lambda_{i-1}} \alpha + \norm{w}_2\\
		&\leq \alpha + \alpha^2(1 + 4\alpha) + 2\alpha^2\\
		&\leq \alpha(1 + 4\alpha),
	\end{align*}
	and that $\norm{w}_2 \leq \alpha^2(1 + 4\alpha)^2 + 2 \alpha^3 \leq 2\alpha^2$ (recall that $\alpha \leq \frac1{20}$).
	The proof now follows since 
	\begin{align*}
		u_{i+1}
		&= v_{i+1} - \ip{u_{i},v_{i+1}} u_i\\
		&= v_{i+1} + \lambda_{i} (v_i + \lambda_{i-1} v_{i-1} + w_i)\\
		&= v_{i+1} + \lambda_{i} v_{i} + w_{i+1}.
	\end{align*}
\end{proof}

\subsection{Algorithms}

Let $\Mc$ be a randomized algorithm that uses $m$ random coins. For $r \in \zo^m$ we denote by $\Mc_r$ the (deterministic) algorithm $\Mc$ after fixing its random coins to $r$.
Given an oracle-aided algorithm $\Ac$ and algorithm $\Bc$, we denote by $\Ac^{\Bc}$ the algorithm $\Ac$ with oracle access to $\Bc$.

\subsection{Differential Privacy}

\begin{definition}[Differential privacy~\citep{DMNS06,DKMMN06}]\label{def:dp} 
    A randomized mechanism $\Mc \colon \cX^n \rightarrow \cY$ is \emph{$(\eps,\delta)$-differentially private} (in short, $(\eps,\delta)$-DP) if for every neighboring databases $X=(x_1,\ldots,x_n), \: X' = (x_1',\ldots,x_n') \in \cX^n$ (i.e., differ by exactly one entry), and every set of outputs $\cT \subseteq \cY$, it holds that
    \begin{align*}
        \pr{\Mc(X) \in \cT} \leq e^{\eps} \cdot \pr{\Mc(X') \in \cT} + \delta
    \end{align*}
\end{definition}

\subsubsection{Basic Facts}

\remove{
\begin{fact}[\cite{BS16}]\label{fact:indis}
	Two random variables $X,X'$ over a domain $\cX$ are $(\eps,\delta)$-indistinguishable, if and only if there exist events $E, E' \subseteq \cX$ with $\pr{X \in E},\pr{X' \in E'} \geq 1-\delta$ such that $X|_{E}$ and $X'|_{E'}$ are $\eps$-indistinguishable.
\end{fact} 
}

\begin{fact}[Post-Processing]\label{fact:post-processing}
    If $\Mc \colon \cX^n \rightarrow \cY$ is $(\eps,\delta)$-DP then for every randomized $\Fc \colon \cY \rightarrow \cZ$, the mechanism $\Fc\circ \Mc \colon \cX^n \rightarrow \cZ$ is $(\eps,\delta)$-DP.
\end{fact}

Post-processing holds when applying the function on the output of the DP mechanism.
In this work we sometimes need to apply the mechanism on the output of a function. While this process does not preserve DP in general, it does so assuming the function is \emph{neighboring-preserving}.  

\begin{definition}[Neighboring-Preserving Algorithm]\label{def:neighbor-preserving}
	We say that a randomized algorithm $\Gc \colon \cX^n \rightarrow \cY^m$ is \emph{neighboring-preserving} if for every neighboring $X, X' \in \cX^n$, the outputs $\Gc(X), \Gc(X') \in \cY^m$ are neighboring w.p. $1$.
\end{definition}

\begin{fact}\label{fact:neighboring-to-neighboring}
	If $\Gc \colon \cX^n \rightarrow \cY^m$ is neighboring-preserving and $\Mc \colon \cY^m \rightarrow \cZ$ is $(\eps,\delta)$-DP, then $\Mc \circ \Gc \colon \cX^n \rightarrow \cZ$ is $(\eps,\delta)$-DP.
\end{fact}

\subsubsection{Zero-Concentrated Differential Privacy (zCDP)}

Our empirical evaluation (\cref{sec:experiments}) is performed in the zCDP model  of \cite{BS16}, defined below.

\begin{definition}[R\'{e}nyi Divergence (\cite{Renyi61})]
	Let $\by$ and $\by'$ be random variables over $\cY$. For $\alpha \in (1,\infty)$, the \emph{R\'{e}nyi divergence} of order $\alpha$ between $\by$ and $\by'$ is defined by 
	\begin{align*}
		D_{\alpha}(\by || \by') = \frac1{\alpha-1} \cdot \ln \paren{\eex{y \la \by}{\paren{\frac{P(y)}{P'(y)}}^{\alpha-1}}},
	\end{align*}
	where $P(\cdot)$ and $P'(\cdot)$ are the probability mass/density functions of $\by$ and $\by'$, respectively.  
\end{definition}

\begin{definition}[zCDP Indistinguishability]\label{def:indis}
	We say that two random variable $\by,\by'$ over a domain $\cY$ are \emph{$\rho$-indistinguishable} (denote by $\by \approx_{\rho} \by'$), iff for every $\alpha \in (1,\infty)$ it holds that
	\begin{align*}
		D_{\alpha}(\by || \by'), D_{\alpha}(\by' || \by) \leq \rho \alpha.
	\end{align*}
	We say that $\by,\by'$ are \emph{$(\rho,\delta)$-indistinguishable} (denote by $\by \approx_{\rho,\delta} \by'$), iff there exist events $E,E' \subseteq \cX$ with $\pr{\by \in E}, \pr{\by' \in E'} \geq 1-\delta$ such that $\by|_{E} \approx_{\rho} \by|_{E'}$.
\end{definition}

\begin{definition}[$(\rho,\delta)$-zCDP \cite{BS16}]\label{def:zCDP}
	A mechanism $\Mc$ is \emph{$\delta$-approximate $\rho$-zCDP} (in short, $(\rho,\delta)$-zCDP), if for any neighboring databases $X,X'$ it holds that $\Mc(X) \approx_{\rho,\delta} \Mc(X')$.
\end{definition}

\paragraph{The Gaussian Mechanism}

\remove{
\begin{definition}[Gaussian distributions]
	For $\mu \in \bbR$ and $\sigma \geq 0$,  let $\cN(\mu,\sigma^2)$ be the Gaussian distribution over $\bbR$ with probability density function $p(z) = \frac1{\sqrt{2\pi}} \exp\paren{-\frac{(z-\mu)^2}{2 \sigma^2}}$.
	For higher dimension $d \in \bbN$, let $\cN(\pt{0},\sigma^2\cdot \bbI_{d \times d})$ be the spherical multivariate Gaussian distribution with variance $\sigma^2$ in each axis. That is, if $Z \sim \cN(\pt{0},\sigma^2\cdot \bbI_{d \times d})$ then $Z = (Z_1,\ldots,Z_d)$ where $Z_1,\ldots,Z_d$ are i.i.d. according to $N(0,\sigma^2)$.
\end{definition}

%\Enote{Follows by ``Basic tail and concentration bounds'' equation (2.7)}
\begin{fact}[Concentration of One-Dimensional Gaussian]\label{fact:one-gaus-concent}
	If $X$ is distributed according to $\cN(0,\sigma^2)$, then for all $\beta > 0$ it holds that $$\pr{X \geq  \sigma \sqrt{2 \ln(1/\beta)}} \leq \beta.$$
\end{fact}
}

\begin{fact}[The Gaussian Mechanism \cite{DKMMN06,BS16}]\label{fact:Gaus}
	Let $\px, \px' \in \bbR^d$ be vectors with $\norm{\px - \px'}_2 \leq \lambda$. For $\rho > 0$, $\sigma = \frac{\lambda}{\sqrt{2\rho}}$ and $Z \sim \cN(\pt{0},\sigma^2 \cdot \bbI_{d \times d})$ it holds that $\px + Z \approx_{\rho} \px' + Z$.
\end{fact}

\subsubsection{FriendlyCore Averaging}\label{sec:prelim:FriendlyCore}

We use the following DP averaging algorithm that given the diameter $\xi$ of a ball that contain most of the points, it can estimate their average.

\begin{fact}[\cite{FriendlyCore22}]\label{fact:FriendlyCoreAvg}
	Let $\lambda \geq 1$ and $\delta \leq \eps, \beta \leq 1$. There exists an $(\eps,\delta)$-DP algorithm $\FCAverage$ that gets as input a dataset $S= (x_1,\ldots,x_n) \in (\bbR^d)^n$ and a parameter $\xi > 0$ and satisfies the following utility guarantee:
	If $n \geq O\paren{\frac{\log(1/\delta)}{\eps} + \frac{\sqrt{d  \log(1/\delta) \log(1/\beta)}}{\lambda \eps}}$ and $\exists S' \subseteq S$ with $\size{S'} \geq 0.8 n$ s.t.  
	$\forall x_i, x_j \in S': \: \norm{x_i - x_j}_2 \leq \xi$, then
	\begin{align*}
		\ppr{y \sim \FCAverage(S,\xi)}{\norm{y - \mu} > \lambda \xi} \leq \beta,
	\end{align*}
	where $\mu = \frac1{\size{S'}}\sum_{x \in S'} x$.
	Furthermore, the running time of $\FCAverage(S,\cdot)$ is $\tilde{O}(d n\log (n/\delta))$ (See Appendix B in \cite{FriendlyCore22}).
\end{fact}

\cref{fact:FriendlyCoreAvg} is not explicitly stated in \cite{FriendlyCore22} since they only analyzed the utility guarantee of their averaging in the zCDP model.
Yet, it can be achieved using similar steps.
First, we need to consider a ``friendly" DP variant of their $\AlgFriendlyAvg$ algorithm (Algorithm 3.3 in \cite{FriendlyCore22}), and as \cite{FriendlyCore22} noted, we can do it such that the probability of failure is low whenever $n \geq O\paren{\frac{\log(1/\delta)}{\eps}}$, and the additive error (given success) decreases in a rate of $\frac{\xi\sqrt{d  \log(1/\delta) \log(1/\beta)}}{n \eps}$, where $\xi$ is the diameter of the points.
\cref{fact:FriendlyCoreAvg} immediately obtained  by combining $\AlgFriendlyAvg$ with their paradigm for DP (Theorem 4.11 in \cite{FriendlyCore22} applied with $\alpha = 0.2$).

\subsubsection{Subspace Recovery Algorithm of \cite{DTTZ14}}\label{sec:prelim:DworkSubspace}

We next describe the subspace recovery algorithm of \cite{DTTZ14} that strongly takes advantage of a large additive gap $\sigma_k^2 - \sigma_{k+1}^2$ for decreasing the noise that is required for privacy.
This algorithm is only used in our empirical evaluation (\cref{sec:experiments}).

\begin{algorithm}[Algorithm 2 in  \cite{DTTZ14}]\label{alg:AnalyzeGauss:DP}
	\item Input: A dataset $X = (x_1,\ldots,x_n) \in (\cS_d)^n$.
	
	\item DP parameters: $\eps,\delta$.
	
	\item Rank parameter: $k$.

	\item Operation:~
	\begin{enumerate}
		\item Compute a projection matrix $\Pi$ to the top-$k$ rows subspace of $X$, and compute the singular values $\sigma_k, \sigma_{k+1}$.
		
		\item Compute $g = \sigma_k^2 - \sigma_{k+1}^2 + Lap(2/\eps)$.\label{step:Dwork:Lap}
		
		\item Compute $W = \Pi + E$, where $E$ is a $d\times d$ symmetric matrix where the upper triangle is i.i.d. samples from $\cN(\pt{0}, \paren{\frac{\Delta_{\eps,\delta}}{g - 2\log(1/\delta)/\eps - 2}}^2)$, where $\Delta_{\eps,\delta} = \frac{1 + \sqrt{2 \log(1/\delta)}}{\eps}$.\label{step:Dwork:Gaus}
		
		\item Output a projection matrix $\tilde{\Pi}$ to the top $k$ eigenvectors of $W$. 
		
	\end{enumerate}
	
\end{algorithm}

Note that when the additive gap $\sigma_k^2 - \sigma_{k+1}^2$ is large, the algorithm will add smaller noise per coordinate in \stepref{step:Dwork:Gaus}.

\begin{fact}[Theorem 11 in \cite{DTTZ14}]
	\cref{alg:AnalyzeGauss:DP} is $(2\eps,2\delta)$-DP.
\end{fact}

The privacy analysis is done by the following steps. First, by the Laplace mechanism, \stepref{step:Dwork:Lap} is $\eps$-DP. Second, by tail bound on the Laplace distribution, the probability that $g - 2 \log(1/\delta)/\eps \leq \sigma_{k}^2 - \sigma_{k+1}^2$ is at least $1 - \delta$. Furthermore, they show that if $\sigma_k^2 - \sigma_{k+1}^2 \geq \alpha$ then the Forbenius-norm sensitivity of the matrix $ \Pi$ is at most $\frac{2}{\alpha - 2}$. So conditioned on the above $1-\delta$ probability event, the Forbenius-norm sensitivity of  $\Pi$ is at most $\frac2{g - 2 \log(1/\delta)/\eps - 2}$, and therefore the Gaussian mechanism step (\cref{step:Dwork:Gaus}) guarantees $(\eps,\delta)$-DP, and by composition the entire process is $(2\eps,2\delta)$-DP.

In order to consider a zCDP version of \cref{alg:AnalyzeGauss:DP}, we replace the Laplace noise $Lap(2/\eps)$ with a Gaussian noise $\cN(0, \sigma^2)$ for $\sigma = \sqrt{2/\rho}$. Given this change, now it holds that $g - \sigma\sqrt{2 \ln(1/\delta)} \leq \sigma_{k}^2 - \sigma_{k+1}^2$ w.p. at least $1-\delta$ (by tail bound on Gaussian distribution). Finally, in \stepref{step:Dwork:Gaus} we replace $\Delta_{\eps,\delta}$ (the required standard deviation for $(\eps,\delta)$-DP) to $1/\sqrt{2 \rho}$ (what is required for $\rho$-zCDP). This results with the following $(2\rho,\delta)$-zCDP algorithm:

\begin{algorithm}[zCDP version of Algorithm 2 in  \cite{DTTZ14}]\label{alg:AnalyzeGauss:zCDP}
	\item Input: A dataset $X = (x_1,\ldots,x_n) \in (\cS_d)^n$.
	
	\item zCDP parameters: $\rho,\delta$.
	
	\item Rank parameter: $k$.

	\item Operation:~
	\begin{enumerate}
		\item Compute a projection matrix $\Pi$ to the top-$k$ rows subspace of $X$, and compute the singular values $\sigma_k, \sigma_{k+1}$.
		
		\item Compute $g = \sigma_k^2 - \sigma_{k+1}^2 + \cN(0, 2/\rho)$.
		
		\item Compute $W = \Pi + E$, where $E$ is a $d\times d$ symmetric matrix where the upper triangle is i.i.d. samples from $\cN(\pt{0}, \paren{\frac{\sqrt{1/(2 \rho)}}{g - 2\sqrt{\ln(1/\delta)/\rho} - 2}}^2)$.
		
		\item Output a projection matrix $\tilde{\Pi}$ to the top $k$ eigenvectors of $W$. 
		
	\end{enumerate}
	
\end{algorithm}

\subsubsection{Lower Bounding Tools from  \cite{PTU24}}\label{sec:PTU24}

\cite{PTU24} showed that for $d = \tilde{\Theta}(n^2)$, the distribution $\cD(n,d)$ below induces a fingerprinting codebook for $n$ users, each codeword is of length $d$.

\begin{definition}[FPC hard distribution $\cD(n,d)$ \cite{PTU24}]\label{def:D}
	Let $\rho$ be the distribution that outputs $p = (e^t-1)/(e^t+1) \in [-1,1]$ for $t \la [-\ln(5n),\ln(5n)]$.
	For $n,d \in \bbN$, let $\cD(n,d)$ be the distribution that chooses independently $p^1,\ldots,p^{d} \sim \rho$, and outputs a codebook $(x_1,\ldots,x_{n}) \in (\oo^{d})^n$ where for each $i \in [n]$ and $j \in [d]$, $x_{i}^j$ is drawn independently over $\oo$ with expectation $p^j$.
\end{definition}

\paragraph{Framework for Lower Bounds}

Consider a mechanism $\Mc \colon \cX^n \rightarrow \cW$ that satisfies some weak accuracy guarantee.
\cite{PTU24} showed that the task of proving a lower bound on $n$ is reduced to the following task:  Transform an FPC codebook $X\in \oo^{n_0 \times d_0}$ into hard instances $Y \in \cX^n$ for $\Mc$, and then extract from the output $w \in \cW$ of $\Mc(Y)$ a vector $q \in \oo^{d_0}$ that is strongly-correlated with $X$ ($n_0$ and $d_0$ are some functions of $n$ and $d$ and the weak accuracy guarantee of $\Mc$), where

\begin{definition}[Strongly Correlated]\label{def:strongly-correlated}
	We say that a random variable $\bold{q} = (\bold{q}^1,\ldots,\bold{q}^d) \in \oo^d$ is \emph{strongly-correlated} with a matrix $X \in \oo^{n \times d}$, if
	\begin{align*}
		\forall b\in \oo, \: \forall j \in \cJ_X^b: \quad \pr{\bold{q}^j = b} \geq 0.9
	\end{align*}
	(recall that $\cJ_X^b = \set{j \in [d] \colon x_i^j = b \text{ for all }i\in[n]}$).
\end{definition}
Denote by $\Gc \colon \oo^{n_0 \times d_0} \times \cZ \rightarrow \cX^n$ the algorithm that generates the hard instances using a uniformly random secret $z \la \cZ$ (i.e., $z$ could be a random permutation, a sequence of random permutations, etc). 
Denote by $\Fc \colon \cZ \times \cW \rightarrow \oo^{d_0}$ the algorithm that extracts a good $q$ using the secret $z$ and the output $w$.
We denote by $\Ac^{\Mc, \Fc, \Gc}(X)$ the entire process:

%\cite{PTU24}'s framework (\cref{lemma:framework}) roughly states that if $\Mc$ is $\paren{1, \frac{\beta}{4n_0}}$-DP and there exists such $\Gc, \Fc$ where: (1) The output of $\Ac^{\Mc, \Fc, \Gc}(X)$ is strongly-correlated with $X$ w.p. at least $\beta$ over the random coins of $\Mc, \Fc, \Gc$, and (2) $\Gc$ is neighboring-preserving (i.e., maps neighboring datasets to neighboring datasets), then $n_0 \geq \Omega\paren{\frac{\sqrt{d_0}}{\log^{1.5}(d_0/ \beta)}}$.

\begin{definition}[Algorithm $\Ac^{\Mc, \Fc, \Gc}$]\label{def:A_MFG}
	Let $\cZ$, $\cW$ be domains, and let $n_0, d_0, n, d \in \bbN$.
	Let $(\Mc, \Fc, \Gc)$ be a triplet of randomized algorithms of types
	$\: \Gc \colon \oo^{n_0 \times d_0} \times \cZ \rightarrow \cX^n$, $\: \Mc \colon \cX^n \rightarrow \cW$, and $\: \Fc \colon \cZ \times \cW \rightarrow [-1,1]^{d_0}$, each uses $m$ random coins. 
	Define $\Ac^{\Mc, \Fc, \Gc} \colon \oo^{n_0 \times d_0} \rightarrow [-1,1]^{d_0}$ as the randomized algorithm that on inputs $X \in \oo^{n_0 \times d_0}$, samples $z \la \cZ$,  $Y \sim \Gc(X,z)$ and $w \sim \Mc(Y)$, and outputs $q \sim \Fc(z,w)$.
	%For $r_M,r_F \in \zo^m$, let $\Ac_{r_M,r_F}^{\Mc, \Fc, \Gc}$ be the randomized algorithm that operates as $\Ac^{\Mc, \Fc, \Gc}$ but uses $r_M$ for the random coins of $\Mc$, and $r_F$ fore the random coins of $\Fc$.
	%For $r_M,r_F \in \zo^m$, let $\Ac_{r_M,r_F}^{\Mc, \Fc, \Gc, \cP}$  be the randomized algorithm that on input $x' \in \oo^{n \times d_0}$, samples $z \sim Z$ and $x \sim \Gc(x', z)$, and outputs $q = \Fc_{r_F}(\Mc_{r_M}(x),z)$, where  $\Mc_{r_M}$ denotes the algorithm $\Mc$ when fixing its random coins to $r_M$, and similarly $\Fc_{r_F}$ is defined. Let $\Ac^{\Mc, \Fc, \Gc, Z}$  be the randomized algorithm that samples $r_F, r_M \la \zo^m$ and operates as $\Ac_{r_F,r_M}^{\Mc, \Fc, \Gc, Z}$ (Equivalently, $\Ac^{\Mc, \Fc, \Gc, Z}(x')$ samples $z \sim Z$ and outputs $q \sim \Fc(\Mc(\Gc(x',z)),z)\:$).
\end{definition}

\begin{definition}[$\beta$-Leaking]\label{def:beta-accurate}
	Let $\Mc, \Fc, \Gc$ be randomized algorithms as in \cref{def:A_MFG}, each uses $m$ random coins, and let $\cD(n_0, d_0)$ be the distribution from \cref{def:D}.
	We say that the triplet $(\Mc, \Fc, \Gc)$ is \emph{$\beta$-leaking} if
	\begin{align*}
		\ppr{r,r', r'' \la \zo^m, \: X \la \cD(n_0,d_0)}{\Ac^{\Mc_{r}, \Fc_{r'}, \Gc_{r''}}(X) \text{ is strongly-correlated with }X} \geq \beta,
	\end{align*}
	where recall that $\Mc_{r}$ denotes the algorithm $\Mc$ when fixing its random coins to $r$ ($\Fc_{r'}, \Gc_{r''}$ are similarly defined).
\end{definition}

\begin{lemma}[Framework for Lower Bounds \cite{PTU24}]\label{lemma:framework}
	Let $\beta \in (0,1]$, $n_0, n,d_0,d \in \bbN$.
	Let $\Mc \colon \cX^n \rightarrow \cW$ be an algorithm such that there exists two algorithms  $\Gc \colon \oo^{n_0 \times d_0} \times \cZ \rightarrow \cX^n$ and $ \Fc \colon \cZ \times \cW \rightarrow [-1,1]^{d_0}$ such that the triplet $(\Mc, \Fc, \Gc)$ is $\beta$-leaking (\cref{def:beta-accurate}). If $\Mc$ is $\paren{1, \frac{\beta}{4n_0}}$-DP and $\Gc(\cdot,z)$ is neighboring-preserving (\cref{def:neighbor-preserving}) for every $z \in \cZ$, then $n_0 \geq \Omega\paren{\frac{\sqrt{d_0}}{\log^{1.5}(d_0/ \beta)}}$.
\end{lemma}

\paragraph{Padding-And-Permuting (PAP) FPC}

The main technical tool of \cite{PTU24} for generating hard instance is to sample a random fingerprinting codebook from $\cD(n,d_0)$, append many $1$-marked and $(-1)$-marked columns, and randomly permute all the columns.

\begin{definition}[$\PAP_{n, d_0, \ell}$]\label{def:PAP}
	Let $\ell, n, d_0 \in \bbN$, and let $d = d_0 + 2\ell$.
	We define $\PAP_{n, d_0, \ell} \colon \oo^{n \times d_0} \times \cP_d \rightarrow \oo^{n \times d}$ as the function that given $X\in \oo^{n \times d_0}$ and a permutation matrix $P \in \cP_d$ as inputs,
	outputs $X' = X'' \cdot P$ (i.e., permutes the columns of $X''$ according to $P$), where $X''$ is the $\oo^{n\times d}$ matrix after appending $\ell$ $1$-marked and $\ell$ $(-1)$-marked columns to $X$ (where recall that a $b$-marked is a column with all entries equal to $b$).
\end{definition}

Note that for every $n, d_0, \ell \in \bbN$ and $P \in \cP_d$, the function $\PAP_{n, d_0, \ell}(\cdot, P)$ is neighboring-preserving (\cref{def:neighbor-preserving}).

\begin{definition}[Strongly Agrees]\label{def:strongly-agree}
	We say that a vector $q = (q^1,\ldots,q^d)$ strongly-agrees with a matrix $X \in \oo^{n \times d}$, if 
	\begin{align*}
		\forall b\in \oo: \quad \size{\set{j \in \cJ_X^b \colon q^j = b}} \geq 0.9 \size{\cJ_X^b}.
	\end{align*}
\end{definition}

The following lemma capture the main technical property of the PAP technique.

\begin{lemma}[\cite{PTU24}]\label{lemma:PAP}
	Let $\ell, n, d_0 \in \bbN$ such that $d = d_0 + 2\ell$. Let $\Mc \colon \oo^{n \times d} \rightarrow [-1,1]^{d}$ be an mechanism that uses $m$ random coins, $\bP \la \cP_d$ (a random variable) and for $X \in \oo^{n \times d_0}$ let $\bY_{X} = \PAP(X,\bP)$. Then for any distribution $\cD$ over $\oo^{n \times d_0}$:
	\begin{align*}
		\lefteqn{\ppr{r \la \zo^m, \: X \sim \cD}{(\Mc_{r}(\bY_{X}) \cdot \bP^T)^{1,\ldots,d_0}\text{ is strongly-correlated with }X}}\\
		&\geq \eex{X \sim \cD}{\pr{\Mc(\bY_{X})\text{ strongly-agrees with } \bY_{X}}}.
	\end{align*}
\end{lemma}

\section{Upper Bounds}\label{sec:UB}

In this section we prove our upper bounds for subspace estimation. In \cref{sec:upper:weak} we prove \cref{thm:intro:upper_bound:weak}, and in \cref{sec:upper:strong} we prove \cref{thm:intro:upper_bound:strong}. Both algorithm share a similar structure that is defined next in \cref{sec:upper:base}.

\subsection{Base Algorithms}\label{sec:upper:base}

Similarly to \cite{SS21}, our algorithms will follow the sample and aggregate approach of \cite{NRS07}. That is, we partition the rows into $t$ subsets, compute (non-privately) the top-$k$ projection matrix of each subset, and then privately aggregate the projections. This is \cref{alg:EstSubspace} that uses oracle access to an aggregation algorithm. Unlike \cite{SS21} who assumed that the rows are i.i.d. Gaussian samples, here we take a \emph{random} partition and show that with large enough probability over the randomness of the partition, the projection matrices are indeed close to each other. 
We consider two types of aggregations: The first type, called \cref{alg:NaiveAgg}, simply treats the matrices as vectors of dimension $d^2$ and computes a DP-average of them using FriendlyCore averaging (\cref{fact:FriendlyCoreAvg}). The second type, called \cref{alg:SSAgg}, is more similar to the aggregation done by  \cite{SS21}. That is, sample reference points $p_1,\ldots,p_q$ and then aggregate the $kd$ dimensional points $\set{(\Pi_j p_1,\ldots,\Pi_j p_q)}_{j=1}^t$. The difference from \cite{SS21} is that we use FriendlyCore averaging (\cref{fact:FriendlyCoreAvg}) which simplifies the construction.

\begin{algorithm}[Algorithm $\EstSubspace$]\label{alg:EstSubspace}
	\item Input: A dataset $X = (x_1,\ldots,x_n) \in (\cS_d)^n$.
	
	\item Parameters: $k, t$.
	
	\item Oracle: A DP algorithm $\Agg$ for aggregating projection matrices.
	
	\item Operation:~
	\begin{enumerate}
		\item Randomly split $X$ into $t$ subsets, each contains (at least) $m = \floor{n/t}$ rows.
		
		Let $X_1,\ldots,X_t$ be the resulting subsets.
		
		\item For each $j \in [t]$: Compute the projection matrix $\Pi_j$ of the top-$k$ rows subset of $X_j$.\label{step:Pis}
		
		\item Output $\tPi \sim \Agg(\Pi_1,\ldots,\Pi_t)$.
		
	\end{enumerate}
	
\end{algorithm}

\begin{algorithm}[Algorithm $\NaiveAgg$]\label{alg:NaiveAgg}
	\item Input: A dataset $\vec{\Pi} = (\Pi_1,\ldots,\Pi_t) \in (\cW_{d,k})^t$. 
	
	\item Privacy parameters: $\eps,\delta \leq 1$.
	
	\item Utility parameter: $\xi \in [0,1]$.
	
	\item Operation:~
	\begin{enumerate}
		
		\item Compute $\hPi \sim \FCAverage_{\eps,\delta}(\vec{\Pi}, \xi)$ (i.e., each $\Pi_j$ is treated like a vector in $\bbR^{d^2}$).
		
		\item Output $\tPi = \argmin_{\Pi \in \cW_{d,k}} \norm{\Pi - \hPi}_F$.
		
	\end{enumerate}
	
\end{algorithm}

\begin{algorithm}[Algorithm $\SSAgg$]\label{alg:SSAgg}
	\item Input: A dataset $(\Pi_1,\ldots,\Pi_t) \in (\cW_{d,k})^t$. 
	
	\item Utility Parameters: $k, q, \xi$.
	
	\item Privacy parameters: $\eps,\delta \leq 1$.
	
	\item Operation:~
	\begin{enumerate}
		
		\item Sample $p_1,\ldots,p_q \sim \cN(\vec{0},\bbI_{d\times d})$ (i.i.d.\ samples from a standard spherical Gaussian). 
		
		\item For $j \in [t]$, compute $y^j = (\Pi_j p_1, \ldots, \Pi_j p_q) \in \bbR^{qd}$, and let $Y = (y^1,\ldots,y^t)$.\label{step:ys}
		
		\item Compute $z = (z_1,\ldots,z_{qd}) \sim \FCAverage_{\eps,\delta}(Y, \xi)$ ($z \in \bbR^{qd}$).\label{step:zs}
		
		\item Let $\tP$ be the $q \times d$ matrix whose \ith row (for $i \in [q]$) is $(z_{d(i-1) + 1}, \ldots,z_{di})$ (which estimates the projection of $p_i$ onto the top-$k$ rows subspace of $X$).\label{step:W}
		%\item Let $W$ be the $q \times d$ matrix whose \ith row (for every $i \in [q]$) is $w_i = (z_i, z_{q+i}, z_{2q + i}, \ldots,z_{(d-1)q + i}) \in \bbR^d$.
		
		\item Output the projection matrix $\tPi$ of the top-$k$ rows subspace of $\tP$.\label{step:output}
		
	\end{enumerate}
	
\end{algorithm}

\subsubsection{Running Time}

We analyze the running time of $\EstSubspace_{k,t}^{\SSAgg_{k,q,\xi}}$.
Denote by $T(n,d,k)$ the running time of computing a projection matrix to the top-$k$ row subspace of an $n \times d$ matrix.
The running time of \stepref{step:Pis} in $\EstSubspace$ is $t \cdot T(n/t,d,k)$. The running time of $\SSAgg$ is $O(dqtk)$ on \stepref{step:ys}, $O(dq t \log t)$ on \stepref{step:zs} (\cref{fact:FriendlyCoreAvg}), and $T(q,d,k)$ on \stepref{step:output}. Overall it is $t\cdot T(n/t,d,k) + T(q,d,k) + O(dq t (\log t + k))$. For both our weak and strong estimators (described next) we use $n/t = \tilde{\Theta}(k)$ and $q = \tilde{O}(k)$, and therefore we obtain that the total running time is $\frac{n}{m} \cdot T(m, d, k) + \tilde{O}(d kn)$ for $m = n/t = \tilde{\Theta}(k)$.
%The running time of \stepref{step:Pis} in $\EstSubspace$ is $t \cdot T(n/t,d,k)$. The running time of $\SSAgg$ is $O(dqtk)$ on \stepref{step:ys}, $O(dq t \log t)$ on \stepref{step:zs} (\cref{fact:FriendlyCoreAvg}), and $T(q,d,k)$ on \stepref{step:output}. Overall it is $t\cdot T(n/t,d,k) + T(q,d,k) + O(dq t (\log t + k))$. For both our weak and strong estimators (described next) we use $n/t = \tilde{\Theta}(k)$ and $q = O(k)$, and therefore we obtain that the total running time is $\frac{n}{m} \cdot T(m, d, k) + \tilde{O}(dkn)$ for $m = n/t = \tilde{\Theta}(k)$.

\subsubsection{Key Property}\label{sec:key-property}

In order to claim that the (non-private) projection matrices are close to each other, we use the following lemma which states that with high enough probability over a random subset, the top-$k$ projection matrix in the subset is close to the top-$k$ projection matrix of the entire matrix. 

\begin{lemma}[Restatement of \cref{lemma:technique:Pi-distance-bounds}]\label{lemma:Pi-distance-bounds}
	Let $X = (x_1,\ldots, x_n) \in (\cS_d)^n$ with singular values $\sigma_1 \geq \ldots \geq \sigma_{\min\set{n,d}}\geq 0$ and $\sigma_k^2 \geq 0.01n/k$. Let $\bX' \in (\cS_d)^m$ be a uniformly random $m$-size subset of the rows of $X$ (without replacement). Let $\Pi$ and $\bPi'$ be the projection matrices to the top-$k$ rows subspace of $X$ and $\bX'$, respectively. 
	Then the following holds for $\gamma_1 = \frac{\sigma_{k+1}}{\sigma_k}$ and $\gamma_2 = \frac{\sqrt{\sum_{i=k+1}^{\min\set{n,d}} \sigma_i^2}}{\sigma_k}$:
	\begin{enumerate}
		
		\item If $m \geq \max\set{800 k \ln\paren{\frac{k}{4\beta}}, \: 2 \gamma_1 n}$, then $\pr{\norm{\Pi - \bPi'} \leq 4 \sqrt{\frac{2n}{m}} \cdot \gamma_1} \geq 1-\beta/2$.\label{item:Pi-Spec-diff}
		
		\item If $m \geq 800 k \ln\paren{\frac{k}{4\beta}}$ and $\beta \geq 4 \gamma_2^2$, then $\pr{\norm{\Pi - \bPi'}_F \leq 4 \sqrt{\frac{2}{\beta}} \cdot \gamma_2} \geq 1-\beta$.\label{item:Pi-F-diff}
	\end{enumerate}
\end{lemma}
\begin{proof}
	We will prove each part of the lemma by applying \cref{cor:proj-diff}.
	In the following, let $X = \sum_{i=1}^n \sigma_i u_i v_i^T$ be the SVD of $X$, and note that we can write $\bX' = (x_{\bi_1},\ldots,x_{\bi_m})$ where $\set{\bi_1,\ldots,\bi_m}$ is a random subset of $[n]$ (without replacement). 
	Let $Y = (y_1,\ldots,y_n) = \sum_{i=1}^k \sigma_i u_i v_i^T$  and let $Z = (z_1,\ldots,z_n) = \sum_{i=k+1}^{\min\set{n,d}}  \sigma_i u_i v_i^T$, and note that 
	$Span\set{y_1,\ldots,y_n} = Span\set{v_1,\ldots,v_k}$ is orthogonal to $Span\set{z_1,\ldots,z_n} = Span\set{v_{k+1},\ldots,v_{\min\set{n,d}}}$.
	Furthermore, define the random matrices $\bY' = (y_{\bi_1},\ldots,y_{\bi_m})$ and $\bZ' = (z_{\bi_1},\ldots,z_{\bi_m})$ and note that $\bX' = \bY' + \bZ'$.
	
	First, by \cref{prop:sigma-concent} (applied on $v_1,\ldots,v_k$) and the assumptions on $m, \sigma_k$, it holds by the union bound that
	\begin{align}\label{eq:sigma-k-bounds}
		\pr{\sigma_k(\bY') \geq \sqrt{\frac{m}{2n}} \sigma_k} \geq 1 - \beta/2.
	\end{align}

	In the following we assume that the event in \cref{eq:sigma-k-bounds} occurs. We first prove \cref{item:Pi-Spec-diff}.
	Note that $\norm{\bZ'} \leq \norm{\bZ} \leq \sigma_{k+1} \leq \gamma_1 \sigma_k$ and that $\sigma_k(\bY')^2 \geq \frac{m}{2n} \sigma_k^2 \geq 2 \gamma_1^2 \sigma_k^2 \geq 2 \norm{\bZ'}^2$ (the second inequality holds since $m \geq 4 \gamma_1^2 n$). By applying \cref{cor:proj-diff}(\ref{item:proj-diff:S-norm}) on $\bX',\bY',\bZ'$ we conclude that
	\begin{align*}
		\norm{\Pi - \bPi'} \leq 4 \cdot \frac{\norm{\bZ'}}{\sigma_k(\bY')} \leq 4 \sqrt{\frac{2n}{m}} \cdot \gamma_1.
	\end{align*}

	We next focus on proving \cref{item:Pi-F-diff}.
	Note that $\norm{Z}_F^2 =  \sum_{i=k+1}^{\min\set{n,d}} \sigma_i^2 = \gamma_2^2 \sigma_k^2$, and that
	$\ex{\norm{\bZ'}_F^2} = \frac{m}{n} \norm{Z}_F^2 = \frac{m}{n} \cdot \gamma_2^2 \sigma_k^2$. Therefore by Markov's inequality
	\begin{align}\label{eq:Z_prime:Markov}
		\pr{\norm{\bZ'}_F^2 \leq \frac{2 m}{\beta n}  \gamma_2^2 \sigma_k^2}  \geq 1 - \beta/2.
	\end{align}
	In the following we assume that the event in \cref{eq:Z_prime:Markov} occurs.
	Note that $\sigma_k(\bY')^2 \geq \frac{m}{2n} \sigma_k^2 \geq \frac{2 m}{\beta n}  \gamma_2^2 \sigma_k^2 \geq 2 \norm{\bZ'}^2$ (the second inequality holds since $\beta \geq 4 \gamma_2^2$). By applying \cref{cor:proj-diff}(\ref{item:proj-diff:F-norm}) on $\bX',\bY',\bZ'$ we conclude that
	\begin{align*}
		\norm{\Pi - \bPi'} \leq 2\sqrt{2} \cdot \frac{\norm{\bZ'}_F}{\sigma_k(\bY')} \leq 2 \sqrt{2} \cdot \frac{\sqrt{\frac{2m}{\beta n}} \gamma_2 \sigma_k}{\sqrt{\frac{m}{2n}} \sigma_k} \leq 4 \sqrt{\frac{2}{\beta}} \cdot \gamma_2.
	\end{align*}
\end{proof}

\subsection{Weak Estimator}\label{sec:upper:weak}

In this section, we prove \cref{thm:intro:upper_bound:weak}, stated below.

\begin{theorem}[Restatement of \cref{thm:intro:upper_bound:weak}]\label{thm:upper_bound:weak}
		Let $n,k,d \in \bbN$, $\lambda > 0$, $\eps,\delta \in (0,1]$ where $k \leq \min\set{n,d}$.  
		There exists an $(k, \lambda, \beta = 0.9, \gamma_{\max} = \Omega(\min\set{\frac1{\lambda}, 1}))$-weak subspace estimator $\Mc \colon (\cS_d)^n \times [0,1] \rightarrow \bbR^{d \times d}$ with
		\begin{align*}
			n = O\paren{k \log k \paren{\frac{\log(1/\delta)}{\eps} + \frac{\min\set{k \sqrt{d}, \: d} \sqrt{\log(1/\delta)}}{\lambda \eps}}}
		\end{align*}
		such that $\Mc(\cdot, \gamma)$ is $(\eps,\delta)$-DP for every $\gamma \in [0,1]$.
\end{theorem}

\cref{thm:upper_bound:weak} is an immediate corollary of the following \cref{lemma:naive:weak,lemma:ss:weak}.
%While the success probability $\beta$ in \cref{lemma:naive:weak,lemma:ss:weak} is just a small constant, it can be easily amplified to $\beta = 0.99$ using a constant number of repetitions (and dividing the privacy budget between them).

\begin{lemma}\label{lemma:naive:weak}
	Let $t = c_1\cdot \paren{\frac{\log(1/\delta)}{\eps} + \frac{d \sqrt{\log(1/\delta) \log(20)}}{\lambda \eps}}$ (where $c_1$ is the hidden constant in \cref{fact:FriendlyCoreAvg}).
	Then for any $n \geq 800 k \ln\paren{25k} \cdot t$,
	the mechanism $\Mc \colon (\cS_d)^n \times [0,1] \rightarrow \cW_{d,k}$ defined by $\Mc(X, \gamma) \eqdef \EstSubspace_{k, t}^{\NaiveAgg_{\eps,\delta, \xi = 60 \gamma}}(X)$ is an $(k, \lambda, \beta=0.9, \gamma_{\max} = \frac1{20})$-weak-subspace-estimator.
\end{lemma}
\begin{proof}
	Let $X \in (\cS_d)^n$ with $\sqrt{\frac{\sum_{i=k+1}^{\min\set{n,d}}\sigma_{i}(X)^2}{\sigma_k(X)^2}} \leq \gamma \leq \gamma_{\max}$ and let $\Pi \in \cW_{d,k}$ be the projection of the top-$k$ rows subspace of $X$. Consider a random execution of $\Mc(X)$. Let $\set{\bPi_j}_{j=1}^t, \bhPi$ be (random variables of) the values of $\set{\Pi_j}_{j=1}^t, \hPi$ in the execution, and let $\btPi$ be the output. By
	\cref{lemma:Pi-distance-bounds}(\ref{item:Pi-F-diff}) (recall that $\gamma_{\max} \leq \frac1{20}$) and the union bound,
	\begin{align}\label{eq:naive-weak:Pi-dist}
		\forall j \in [t]: \quad \pr{\norm{\Pi - \bPi_j} \leq 60 \gamma} \geq 0.99,
	\end{align}
	Let $\ba_j = \indic{\norm{\Pi - \bPi_j} \leq 60 \gamma}$ (indicator random variable) and let $\ba = \sum_{j=1}^t \ba_j$. By \cref{eq:naive-weak:Pi-dist} it holds that $\ex{\ba} \geq 0.99 t$, and recall that $\ba \leq t$. It follows that
	\begin{align}\label{eq:naive-weak:a}
		\pr{\ba \geq 0.8 t} \geq \frac{\ex{\ba} - 0.8 t \cdot \pr{\ba < 0.8t}}{t} \geq 0.99 - 0.8(1 - \pr{\ba \geq 0.8 t} ) \:\: \implies \:\: \pr{\ba \geq 0.8 t} \geq 0.95.
	\end{align}

	In the following we assume that the event $\ba \geq 0.8 t$ occurs. Let $\bJ = \set{j \in [t] \colon \ba_j = 1}$. 
	Note that our choice of $t$ satisfies
	\begin{align*}
		t \geq c' \cdot \paren{\frac{\log(1/\delta)}{\eps} + \frac{d \sqrt{\log(1/\delta) \log(100)}}{\paren{\frac{\lambda}{500}}\cdot \eps}}.
	\end{align*}
	where $c'$ denotes the constant from \cref{fact:FriendlyCoreAvg}.
	Therefore we conclude by \cref{fact:FriendlyCoreAvg} (FriendlyCore averaging) that
	\begin{align*}
		\pr{\norm{\bPi - \bhPi}_F \leq \frac{\lambda \gamma}{4}} \geq 0.99.
	\end{align*} 
	The proof of the lemma now follows by \cref{prop:F-dist-implies-estimator} since $\norm{\bPi - \btPi}_F \leq 2 \norm{\bPi - \bhPi}_F$.
\end{proof}

\begin{lemma}\label{lemma:ss:weak}
	Let $c_1$, $c_2$, $c_3$ be the constants from \cref{fact:SS-Pips,fact:SS-pi-span,fact:SS-pi-span} (respectively), and let $c$ be a large enough constant.
	Let $t = c\cdot \paren{\frac{\log(1/\delta)}{\eps} + \frac{\paren{\sqrt{k} + \sqrt{\log\frac{dk\log(1/\delta)}{\lambda \eps}}}\sqrt{kd\log(1/\delta)}}{\lambda \eps}}$, $\: \eta = c_1 \cdot \paren{\sqrt{k} + \sqrt{\ln(qt)}}$ and $q = c_2\cdot k$.
	Then for any $n \geq 800 k \ln\paren{25k} \cdot t$,
	the mechanism $\Mc \colon (\cS_d)^n \times [0,1] \rightarrow \cW_{d,k}$ defined by $\Mc(X, \gamma) \eqdef \EstSubspace_{k, t}^{\SSAgg_{\eps,\delta, k,q,\xi = 120 \eta \sqrt{k} \gamma}}(X)$ is an $(k, \lambda, \beta=0.9, \gamma_{\max} = \Theta(\min\set{\frac1{\lambda}, 1}))$-weak-subspace-estimator.
\end{lemma}
\begin{proof}
	Let $X \in (\cS_d)^n$ with $\sqrt{\frac{\sum_{i=k+1}^{\min\set{n,d}}\sigma_{i}(X)^2}{\sigma_k(X)^2}} \leq \gamma \leq \gamma_{\max}$ and let $\Pi \in \cW_{d,k}$ be the projection of the top-$k$ rows subspace of $X$. Consider a random execution of $\Mc(X)$. Let $\set{\bPi_j}_{j=1}^t, \set{\bp_i}_{i=1}^q,  \set{\by^j}_{j=1}^{t}, \bz, \btP$ be (random variables of) the values of $\set{\Pi_j}_{j=1}^t, \set{p_i}_{i=1}^q, \set{\y^j}_{j=1}^{t}, z, \tP$ in the execution, and let $\btPi$ be the output.
	As in the proof of \cref{lemma:naive:weak}, let $\ba = \sum_{j=1}^t \ba_j$ where $\ba_j = \indic{\norm{\Pi - \bPi_j} \leq 60 \gamma}$. Then \cref{eq:naive-weak:a} imples that
	\begin{align*}
		\pr{\ba \geq 0.8 t} \geq 0.95.
	\end{align*}
	In the following we assume that the event $\ba \geq 0.8 t$ occurs. Let $\bJ = \set{j \in [t] \colon \ba_j = 1}$. 
	By \cref{fact:SS-Pips} and the definition of $\eta$ we obtain that
	\begin{align*}
		\pr{\forall i\in [q], j \in \bJ: \: \norm{\paren{\Pi - \bPi_{j}}\bp_i}_2 \leq 60 \eta  \gamma} \geq 0.95,
	\end{align*}
	In the following we assume that the above event occurs.
	This yields that 
	\begin{align*}
		\forall i ,j \in \bJ: \: \norm{\by_i - \by_j}_2 \leq 120\eta \sqrt{k} \gamma = \xi.
	\end{align*}
	Furthermore, by the definition of $t$ and $\eta$ (note that $\eta$ depends on $\log(t)$), it holds that
	\begin{align*}
		t \geq c' \cdot \paren{\frac{\log(1/\delta)}{\eps} + \frac{\sqrt{dk \log(1/\delta)}}{\paren{\frac{c_3 \lambda}{1500 \eta}} \cdot \eps}},
	\end{align*}
	where $c'$ denotes the constant from \cref{fact:FriendlyCoreAvg}. Therefore we obtain by \cref{fact:FriendlyCoreAvg} (FriendlyCore averaging) that
	\begin{align*}
		\pr{\norm{\bP - \btP}_F \leq \underbrace{\frac{c_3 \lambda \sqrt{k} \gamma}{12}}_{\alpha}} \geq 0.99,
	\end{align*}
	where $\bP$ is the $q \times d$ matrix whose rows are $\Pi \bp_1,\ldots, \Pi \bp_q$.
	By \cref{fact:SS-pi-span} we have that
	\begin{align*}
		\pr{\sigma_k(\bP) \geq c_3 \sqrt{k}} \geq 0.95,
	\end{align*}
	and in the following we assume that the above event occurs (which in particular implies that $Span(\bP^T) = Span(\Pi)$).
	Finally, since $2\alpha \leq \sigma_k(\bP)$ by assumption (and assuming $\gamma_{\max} \leq \frac{6}{c_3 \lambda}$), we conclude by \cref{prop:P-Pprime} that
	\begin{align*}
		\norm{\bPi - \btPi}_F \leq 2\sqrt{2} \cdot \frac{\alpha}{\sigma_k(P) - \alpha} \leq \lambda \gamma/2.
	\end{align*}
	We therefore conclude the proof of the lemma by \cref{prop:F-dist-implies-estimator}.
	
\end{proof}

\subsection{Strong Estimator}\label{sec:upper:strong}

In this section, we prove \cref{thm:intro:upper_bound:strong}, stated below.

\begin{theorem}[Restatement of \cref{thm:intro:upper_bound:strong}]\label{thm:upper_bound:strong}
	Let $n,k,d \in \bbN$, $\lambda > 0$, such that $k \leq \min\set{n,d}$.  
	There exists an $(k, \lambda, \beta = 0.8, \gamma_{\max})$-weak subspace estimator $\Mc \colon (\cS_d)^n \times [0,1] \rightarrow \bbR^{d \times d}$ with 
	\begin{align*}
		\gamma_{\max} = \Omega\paren{\min\set{\frac1{\lambda}, \: \frac{\lambda^2 \eps^2}{\lambda^2 \eps \log(1/\delta) + \paren{k + \log\paren{\frac{d k \log(1/\delta)}{\lambda \eps}}}  dk \log(1/\delta)}}}
	\end{align*}
	and 
	\begin{align*}
		n = O\paren{k \log k \paren{\frac{\log(1/\delta)}{\eps} + \frac{\paren{k + \log\paren{\frac{d k \log(1/\delta)}{\lambda \eps}}}  dk \log(1/\delta)}{\lambda^2 \eps^2}}}
	\end{align*}
	such that $\Mc(\cdot, \gamma)$ is $(\eps,\delta)$-DP for every $\gamma \in [0,1]$.
\end{theorem}

\begin{lemma}
	Let $c_1$, $c_2$, $c_3$ be the constants from \cref{fact:SS-Pips,fact:SS-pi-span,fact:SS-pi-span} (respectively), and let $c$ be a large enough constant.
	Let $t = c\cdot \paren{\frac{\log(1/\delta)}{\eps} + \frac{\paren{k + \log\paren{\frac{d k \log(1/\delta)}{\lambda \eps}}}  dk \log(1/\delta)}{\lambda^2 \eps^2}}$, $\eta = c_1 \cdot \paren{\sqrt{k} + \sqrt{\ln(qt)}}$ and $q = c_2\cdot k$.
	Then for any $n \geq 800 k \ln\paren{25k} \cdot t$,
	the mechanism $\Mc \colon (\cS_d)^n \times [0,1] \rightarrow \cW_{d,k}$ defined by $\Mc(X, \gamma) \eqdef \EstSubspace_{k, t}^{\SSAgg_{\eps,\delta, k, q, \xi = 8 \sqrt{2t k} \eta \gamma}}(X)$ is an $(k, \lambda, \beta=0.8, \gamma_{\max} = \min\set{\frac1{2t}, \frac{6}{c_3 \lambda}})$-strong-subspace-estimator.
\end{lemma}
\begin{proof}
	Fix $X \in (\cS_d)^n$ with $\frac{\sigma_{k+1}(X)}{\sigma_k(X)} \leq \gamma \leq \gamma_{\max}$ and let $\Pi \in \cW_{d,k}$ be the projection of the top-$k$ rows subspace of $X$. Consider a random execution of $\Mc(X)$. Let $\set{\bPi_j}_{j=1}^t, \set{\bp_i}_{i=1}^q,  \set{\by^j}_{j=1}^{t}, \bz, \btP$ be (random variables of) the values of $\set{\Pi_j}_{j=1}^t, \set{p_i}_{i=1}^q, \set{\y^j}_{j=1}^{t}, z, \tP$ in the execution, and let $\btPi$ be the output. 
	By \cref{lemma:Pi-distance-bounds}(\ref{item:Pi-Spec-diff}) (recall that $\gamma_{\max} \leq \frac1{2t}$) and the union bound,
	\begin{align}\label{eq:ss-strong:Pi-dist}
		\forall j \in [t]: \quad \pr{\norm{\Pi - \bPi_j} \leq 4 \sqrt{2t} \gamma} \geq 0.99.
	\end{align}
	Let $\ba_j = \indic{\norm{\Pi - \bPi_j} \leq 4 \sqrt{2t} \gamma}$ (indicator random variable) and let $\ba = \sum_{j=1}^t \ba_j$.
	As in \cref{eq:naive-weak:a}, the above yields that
	\begin{align}\label{eq:ss-strong}
	 	\pr{\ba \geq 0.8 t} \geq 0.95.
	\end{align}
	In the following we assume that the event $\ba \geq 0.8 t$ occurs. Let $\bJ = \set{j \in [t] \colon \ba_j = 1}$. 
	By \cref{fact:SS-Pips} and the definition of $\eta$ we obtain that
	\begin{align*}
		\pr{\forall i\in [q], j \in \bJ: \: \norm{\paren{\Pi - \bPi_{j}}\bp_i}_2 \leq4 \sqrt{2t} \eta \gamma} \geq 0.95,
	\end{align*}
	In the following we assume that the above event occurs.
	This yields that 
	\begin{align*}
		\forall i ,j \in \bJ: \: \norm{\by_i - \by_j}_2 \leq 8 \sqrt{2t k} \eta \gamma = \xi.
	\end{align*}
	Furthermore, by the definition of $t$ and $\eta$, it holds that
	\begin{align*}
		t \geq c' \cdot \paren{\frac{\log(1/\delta)}{\eps} + \frac{\sqrt{dk \log(1/\delta)}}{\paren{\frac{c_3 \lambda}{150 \eta \sqrt{t}}} \cdot \eps}},
	\end{align*}
	where $c'$ denotes the constant from \cref{fact:FriendlyCoreAvg}. Therefore we obtain by \cref{fact:FriendlyCoreAvg} (FriendlyCore averaging) that
	\begin{align*}
		\pr{\norm{\bP - \btP}_F \leq \underbrace{\frac{c_3 \lambda \sqrt{k} \gamma}{12}}_{\alpha}} \geq 0.99,
	\end{align*}
	where $\bP$ is the $q \times d$ matrix whose rows are $\Pi \bp_1,\ldots, \Pi \bp_q$.
	By \cref{fact:SS-pi-span} we have that
	\begin{align*}
		\pr{\sigma_k(\bP) \geq c_3 \sqrt{k}} \geq 0.95,
	\end{align*}
	and in the following we assume that the above event occurs (which in particular implies that $Span(\bP^T) = Span(\Pi)$).
	Finally, since $2\alpha \leq \sigma_k(\bP)$ by assumption (and using $\gamma_{\max} = \frac{6}{c_3 \lambda}$), we conclude by \cref{prop:P-Pprime} that
	\begin{align*}
		\norm{\bPi - \btPi}_F \leq 2\sqrt{2} \cdot \frac{\alpha}{\sigma_k(P) - \alpha} \leq \lambda \gamma/2.
	\end{align*}
	We therefore conclude the proof of the lemma by \cref{prop:F-dist-implies-estimator}.
\end{proof}

\section{Lower Bounds}\label{sec:LB}

In this section, we prove our lower bounds.
In \cref{sec:LB:weak} we prove \cref{thm:intro:lower_bound:weak} (lower bound for weak estimators) and in \cref{sec:LB:strong} we prove \cref{thm:intro:lower_bound:strong} (lower bound for strong estimators).
Both lower bounds rely on the framework of \cite{PTU24}, described in \cref{sec:PTU24}.

Throughout this section, recall that for $d,k \in \bbN$ we denote by $\cW_{d,k}$ the set of all $d \times d$ projection matrices of rank $k$, and denote by $\cP_d$ the set of all $d \times d$ permutation matrices.

\subsection{Weak Estimators}\label{sec:LB:weak}

\begin{theorem}[Restatement of \cref{thm:intro:lower_bound:weak}]\label{thm:lower_bound:weak}
	Let $n,k,d \in \bbN$, $\lambda \geq 1$, $\beta \in (0,1]$ such that $d \geq c k$ and $\lambda^2 \leq \frac{d}{c k \log k}$ for large enough constant $c > 0$, and $n$ is a multiple of $k$.
	If $\Mc \colon (\cS_d)^n \times [0,1] \rightarrow \cW_{d,k}$ is a \emph{$(k, \lambda, \beta, \gamma_{\max} = \frac1{10^6 \lambda^2})$-weak subspace estimator} and $\Mc(\cdot, \gamma)$ is $\paren{1,\frac{\beta}{5nk}}$-DP for every $\gamma \in [0,1]$, then $n \geq \Omega\paren{\frac{\sqrt{k d}/\lambda}{\log^{1.5}(\frac{dk}{\lambda \beta})}}$.
\end{theorem}

\cref{thm:intro:lower_bound:weak} is an immediate corollary of \cref{lemma:framework} (Framework for lower bounds) and the following \cref{lemma:LB:weak}.

\begin{lemma}\label{lemma:LB:weak}
	Let $n,k,d,\lambda,\beta$ and $\Mc$ as in \cref{thm:lower_bound:weak}. Let $\alpha =  \frac1{5\cdot 10^6 \lambda^2 k}$, $\: n_0 = n/k$, $\: \ell = 2\cdot \ceil{\frac14(1-\alpha) d}$, $\; d_0 = d - 2\ell$, $\: \cX = \cS_d$, $\: \cZ = [k] \times (\cP_d)^k$ and $\: \cW = \cW_{d,k}$. Let $\Gc\colon \oo^{n_0 \times d_0} \times \cZ \rightarrow \cX^n$ be \cref{alg:G}, and let $\Fc \colon \cZ \times \cW \rightarrow [-1,1]^{d_0}$ be \cref{alg:F}. Then the triplet $(\Mc, \Fc, \Gc)$ is $\frac{0.8 \beta}{k}$-leaking (\cref{def:beta-accurate}).
\end{lemma}

Note that by \cref{lemma:framework,lemma:LB:weak}, we obtain that $n_0 \geq \Omega\paren{\frac{\sqrt{d_0}}{\log^{1.5}(d_0/ \beta)}}$, but since $n_0 = n/k$ and $d_0 = \Theta(\alpha d) = \Theta\paren{\frac{d}{\lambda^2 k}}$, the proof of \cref{thm:lower_bound:weak} follows.
We next prove \cref{lemma:LB:weak}.

\begin{algorithm}[Algorithm $\Gc$]\label{alg:G}
	\item Parameters: $n_0,n,d_0,d,\ell \in \bbN$.
	
	\item Inputs: $z = (s, (P_1,\ldots,P_k))$ for $s \in [k]$ and $P_1,\ldots,P_k \in \cP_d$, and 
	a matrix $X \in \oo^{n_0 \times d_0}$.
	
	\item Operation:~
	\begin{enumerate}
		
		\item Sample $A = (A_1,\ldots,A_k) \sim \cD(n_0,d_0)^k$, and set $A_s = X$.
		
		\item For $t \in [k]$, compute $B_t = \PAP_{n_0,d_0,\ell}(A_t, P_t) \in \oo^{n_0 \times d}$ (\cref{def:PAP}), and let $B \in \oo^{n \times d}$ be the vertical concatenation of $B_1,\ldots,B_k$.
		
		\item Output $Y = \frac1{\sqrt{d}} B \in (\cS_d)^n$.
		
	\end{enumerate}
	
\end{algorithm}

\begin{algorithm}[Algorithm $\Fc$]\label{alg:F}
	\item Parameters: $n_0,n,d_0,d,\ell \in \bbN$.
	
	\item Inputs: $z = (s, (P_1,\ldots,P_k))$ for $s \in [k]$ and $P_1,\ldots,P_k \in \cP_d$, and 
	a rank-$k$ projection matrix $\tPi \in \cW$ (which is the output of $\Mc(Y, \: \gamma = \frac1{1000 \lambda})$ ).
	
	\item Operation:~
	\begin{enumerate}
		
		\item Compute a vector $u = (u^1,\ldots,u^d) \in Span(\tPi \cdot P_s^T)$ that maximizes $\min\set{\sum_{j=d_0+1}^{d_0 + \ell/2} \sign(u^j), \: -\sum_{j= d_0 + \ell + 1}^{d_0 + 3\ell/2} \sign(u^j)}$.\label{step:F:u}
		
		\item Output $q = \sign(u)^{1,\ldots, d_0} \in \oo^{d_0}$.
		
	\end{enumerate}
	
\end{algorithm}

\subsubsection{Proving \cref{lemma:LB:weak}}

In the following, we define random variables $\bX \sim \cD(n_0, d_0)$ (\cref{def:D}) and $\bz = (\bs, (\bP_1,\ldots,\bP_k)) \la \cZ$, and consider a random execution of $\Ac^{\Mc,\Fc,\Gc}(\bX) = \Fc(\bz, \Mc(\Gc(\bX, \bz)))$.
Let $\bB_{1},\ldots,\bB_k ,\bB, \bY$ be the values of $B_1,\ldots,B_k \in \oo^{n_0 \times d}$, $B \in \oo^{n \times d}$ and $Y \in (\cS_d)^n$ in the execution of $\Gc$, and let $\bY_1 = \frac1{\sqrt{d}} \bB_1, \ldots, \bY_k = \frac1{\sqrt{d}} \bB_k$ (note that $\bY$ is a vertical concatenation of $\bY_1,\ldots,\bY_k$).
Let $\bu$ be the value of $u \in \cS_d$ in the execution of $\Fc$.
For $b \in \oo$ let $\bF_b$ be the set of $b$-marked columns of $\bB_{\bs} \cdot \bP_s^T$ (note that $\bF_1$ includes $d_0+1,\ldots,d_0+\ell$ and $\bF_{-1}$ includes $d_0+ \ell + 1,\ldots,d$). Let $\cH_1 = \set{d_0+1, \ldots, d_0 + \ell/2} \subseteq \bF_1$ and $\cH_{-1} = \set{d_0 + \ell + 1, \ldots,d_0 + 3\ell/2} \subseteq \bF_{-1}$, and let $\cH = \cH_1 \cup \cH_{-1}$.
For $t \in [k]$, define 
\begin{align}\label{def:vt}
	\bv_t = \frac1{\sqrt{2\ell}}\cdot (\underbrace{0\dots,0}_{d_0},\underbrace{1,\ldots \ldots \ldots ,1}_{\ell},\underbrace{-1,\ldots \ldots \ldots ,-1}_{\ell})\cdot \bP_t \: \in \: \cS_d
\end{align}

The following claim holds under our assumption that $\lambda^2 \leq \frac{d}{c k \log k}$ for large enough constant $c$.
\begin{claim}\label{claim:large_spec_gap:weak}
	Let $\gamma =  \frac1{1000 \lambda}$. It holds that
	\begin{align*}
		\pr{\sum_{i=k+1}^n \sigma_i^2(\bY) \leq \gamma^2 \cdot \sigma_k^2(\bY)} \geq 0.9.
	\end{align*}
\end{claim}

\begin{proof}
	
	Recall that $d = d_0 + 2\ell$ and $\ell \geq \frac12(1-\alpha)d$ for $\alpha = \frac1{5 \cdot 10^6 \lambda^2 k}$.
	First, note that 
	\begin{align}\label{eq:vi_indep}
		\pr{\bv_1,\ldots,\bv_k\text{ are linearly independent}} \geq 1 - \paren{2^{-2\ell} + 2^{-2\ell+1} + \ldots + 2^{-2\ell + (k-2)}} \geq 0.99,
	\end{align}
	where the last inequality holds since $\ell \approx d/2 \geq c \cdot k/2$ for large enough constant $c>0$.
	Furthermore, note that for every $s,t \in [k]$, $2\ell\cdot \ip{\bv_s, \bv_t} = \sum_{j \colon \sign(\bv_s^j) = 1} \sign(\bv_t^j) - \sum_{j \colon \sign(\bv_s^j) = -1} \sign(\bv_t^j)$ where each sum has Hypergeometric distributions $\HG_{2\ell, 0, \ell}$ (\cref{def:Hyper}). Therefore by \cref{fact:hyperHoeffding} and the union bound, it holds that
	\begin{align}\label{eq:vs_vt}
		\pr{\forall s,t \in [k]: \abs{\ip{\bv_s, \bv_t}} \leq O\paren{\sqrt{\frac{\log k}{d}}}} \geq 0.99.
	\end{align}
	(i.e., $\bv_1,\ldots,\bv_k$ are almost orthogonal).
	
	In the following, we assume that the events in \cref{eq:vi_indep,eq:vs_vt} occur.
	Using the Gram–Schmidt process on $\bv_1,\ldots,\bv_k$, we obtain orthogonal basis $\bu_1,\ldots,\bu_k$ to $Span\set{\bv_1,\ldots,\bv_k}$ such that
	for every $t \in [k]$, $\bu_t = \bv_t + \lambda_{t-1} \bv_{t-1} +  \bw_t$, where $\size{\lambda_{t-1}} \leq O\paren{\sqrt{\frac{\log k}{d}}}$ and $\norm{\bw_t}_2 \leq O\paren{\frac{\log k}{d}}$ (holds by \cref{prop:gram-schmidt}).
	Recall that $\bY$ is a vertical concatenation of $\frac1{\sqrt{d}} \bB_1,\ldots, \frac1{\sqrt{d}}\bB_k$ and for every $t \in [k]$, the rows of $\bB_t$ are all in 
	\begin{align*}
		\oo^{d_0} \times  (\underbrace{1,\ldots   ,1}_{\ell},\underbrace{-1,\ldots   ,-1}_{\ell}) \cdot \bP_t
		\: = \oo^{d_0} \times  (\underbrace{0,\ldots  \ldots ,0}_{2\ell}) \cdot \bP_t + \sqrt{\frac{2\ell}{d}}\cdot \bv_t 
	\end{align*}
	
	Therefore, we obtain that
	\begin{align}\label{eq:ui}
		\norm{\bY \cdot \bu_t}_2^2 
		&\geq \norm{\bY_t \cdot \bu_t}_2^2\\
		&\geq \paren{\ip{\sqrt{\frac{2\ell}{d}} \bv_t, \bu_t}^2 - \frac{d_0}{d}} \cdot n_0\nonumber\\
		&\geq \paren{(1 - \alpha) \cdot \paren{1 - O\paren{\frac{\log k}{d}}}^2 - \alpha} \cdot n_0\nonumber\\
		&\geq (1 - 4\alpha) \cdot n_0,\nonumber
	\end{align}
	where the last inequality holds whenever $\alpha \geq \Theta(\log k/d)$, which holds by the assumption on $\lambda$.
	We therefore obtain that $\sigma_1^2(\bY), \ldots, \sigma_k^2(\bY) \geq (1-4\alpha)\cdot n_0$ which yields that $\sum_{i=k+1}^n \sigma_i^2(\bY) \leq n - k (1-4\alpha)\cdot n_0 = 4\alpha n$. Hence
	\begin{align}\label{eq:gap}
		\frac{\sum_{i=k+1}^n \sigma_i^2(\bY)}{\sigma_k^2(\bY)} \leq \frac{4\alpha n}{(1-4\alpha)\cdot \frac{n}{k}} = \frac{4 \alpha k}{1 - 4\alpha} \leq \frac1{10^6 \lambda^2} = \gamma^2,
	\end{align}
	where the second inequality holds since $\alpha = \frac{1}{5 \cdot 10^6 \lambda^2 k}$.
	
\end{proof}

The following claim holds under our assumption that $d \geq c k$ for large enough constant $c > 0$.
\begin{claim}\label{claim:u_strongly_agrees}
	It holds that
	\begin{align*}
		\pr{\sign(\bu)^{[d] \setminus \cH}\text{ strongly-agrees with }\paren{\bB_{\bs} \bP_{\bs}^T}^{[d] \setminus \cH}} \geq \frac{0.8 \beta}{k}.
	\end{align*}
	where ``strongly-agrees" is according to \cref{def:strongly-agree}.
\end{claim}

\begin{proof}

In the following we assume that the $0.9$ probability event in \cref{claim:large_spec_gap:weak} occurs.
Since $\Mc$ is $(k,\lambda,\beta, \gamma_{\max} = \frac1{1000 \lambda})$-subspace estimator, it follows from \cref{eq:gap} that w.p. $\beta$, the output $\btPi$ of $\Mc(\bY)$ satisfy 
\begin{align}\label{eq:M_out}
	\norm{\btPi \cdot \bY^T}_F^2 \geq \norm{\bPi \cdot \bY^T}_F^2 - \frac{n}{1000},
\end{align}
where we denote by $\bPi$ the projection matrix onto $Span\set{\bv_1,\ldots,\bv_k}$ (defined in \cref{def:vt}).
In the following, we assume that the event in (\ref{eq:M_out}) occurs.
This yields that there must exists $\bt \in [k]$ such that
\begin{align*}
	\norm{\btPi \cdot \bY_{\bt}^T}_F^2
	&\geq \norm{\bPi \cdot \bY_{\bt}^T}_F^2 - \frac{n}{1000 k}\\
	&\geq \frac{2\ell}{d} \cdot \norm{\bv_{\bt}}_2^2 \cdot \frac{n}{k} - \frac{n}{1000 k}\\
	&\geq (0.999 - \alpha) \cdot \frac{n}{k}
\end{align*}

Since $\bs$ (part of $\pz$) is chosen at random and does not change the distribution of $\bY$ (the input of the mechanism), w.p. $1/k$ the above holds for $\bt = \bs$, i.e.,
\begin{align}\label{eq:event-t-s}
	\norm{\btPi \cdot \bY_{\bs}^T}_F^2 \geq (0.999 - \alpha) \cdot \frac{n}{k}.
\end{align}

In the following we assume that the $1/k$-probability event in \cref{eq:event-t-s} occurs.

%Recall that $X_t = \sqrt{d} \cdot \tX_t$ and $\tPi_t = \tPi \cdot P_t^T$, and let $X''_t = X_t \cdot P_t^T$.
Recall that $\bB_{\bs} = \sqrt{d} \cdot \bY_{\bs}$,  and let $\btPi_{\bs} = \btPi \cdot \bP_{\bs}^T$ and  $\bB'_{\bs} = \bB_{\bs} \cdot \bP_{\bs}^T$.
It follows that
\begin{align}
	\norm{\btPi_{\bs} \cdot (\bB'_{\bs})^T}_F^2 = d \cdot \norm{\btPi \cdot \bY_{\bs}^T}_F^2
	\geq (0.999 - \alpha) \cdot \frac{d n}{k}
\end{align}

In the following, define

\begin{align}\label{eq:v}
	\bv = (\bv^1,\ldots,\bv^d)\in \set{-1,0,1}^{d} \quad \text{where} \quad \bv^j = \begin{cases} 1 & j \in \bF_1 \\ -1 & j \in \bF_{-1} \\ 0  & \text{o.w.} \end{cases}.
\end{align}

Note that each row $i$ of $\bB_{\bs}'$ can be written as $\pv + \bold{\xi}_i$ where $\xi_i \in \set{-1,0,1}^{d_0} \times (\underbrace{0,\ldots,0}_{2\ell})$. 
This yields that 
\begin{align*}
	\norm{\btPi_{\bs} \cdot \bv^T}_2^2
	&= \frac{k}{n} \cdot \norm{\btPi_{\bs} \cdot (\bB_{\bs}')^T}_F^2 - \frac{k}{n} \sum_{i=1}^{n/k} \norm{\btPi_{\bs} \cdot \xi_i^T}_2^2\\
	&\geq  (0.999 - \alpha) d - d_0 \geq (0.999 - 2\alpha)d
\end{align*}

Now, since 
\begin{align*}
	\norm{\bv}_2^2 = \ip{\bv,\bv} = \ip{\btPi_t \bv^T + (I - \btPi_t) \bv^T, \btPi_t \bv^T + (I - \btPi_t) \bv^T} = \norm{\btPi_t \bv^T}_2^2 + \norm{(I - \btPi_t) \bv^T}_2^2,
\end{align*}
we conclude that for $\btv^T = \btPi_t \bv^T \in Span(\btPi_t)$:
\begin{align}\label{eq:exists_good_tv}
	\norm{\bv - \btv}_2^2  = \norm{\bv}_2^2 -  \norm{\btPi_t \bv^T}_2^2  \leq d - (0.999 - 2\alpha) d \leq \frac{d}{500}.
\end{align}

We next define \emph{$(\cI, \eta)$-good} vectors.

\begin{definition}
	We say that a vector $\bw \in \oo^{d}$ is $(\cI,\eta)$-good iff for $(1-\eta)$ fraction of the indices $j \in \cI$ it holds that
	$\sign(\bw^j) = \sign(\bv^j)$ (for the $\bv$ defined in \cref{eq:v}).
\end{definition}

We use the following trivial fact:
\begin{observation}\label{claim:agrees}
	If $\bw$ is not $(\cI,\eta)$-good, then $\norm{\bw-\bv}_2^2  \geq \eta \size{\cI}$.
\end{observation}

Since for both $b \in \oo$, $\size{\cH_b} \geq \frac14 (1 - \alpha) d \geq d/5$ and $\norm{\bv - \btv}_2^2  \leq \frac{d}{500}$, \cref{claim:agrees} implies that $\btv$ is $(\cH_{1}, \eta)$-good and $(\cH_{-1}, \eta)$-good for $\eta = \frac1{100}$. Therefore, the vector $\bu$ (computed in $\Fc$) is also $(\cH_{1}, \eta)$-good and $(\cH_{-1}, \eta)$-good.

In the following, we use similar arguments to \cite{DTTZ14} for claiming that because $\bu$ is good on half of the padding location, then it should also be good on the rest of the marked locations.

From the point of view of the algorithm $\Mc$ (which does not know $\bP_{\bs}$), the locations in $\cH_b$ are indistinguishable from those in $\bF_b \setminus \cH_b$. 
Therefore, for any point that is not $(\bF_b, 3\eta)$-good, the probability (taken over the random choice of $\bP_{\bs}$) that it is $(\cH_b, 2\eta)$-good  is at most $\exp(-\Omega(\eta^2 d))$.
Now let $\bN$ be an $(\eta \sqrt{d})$-net of the $\sqrt{d}$-sphere in $Span(\btPi_{\bs})$ (\cref{def:net}).
Taking a union bound over a $\exp\paren{O(k \log\paren{1/\eta})} = \exp(O(k))$ points in $\bN$ (\cref{fact:net}), and recall that $d \geq c k$ for large enough constant $c$, we conclude that except with probability $\exp(O(k))\cdot \exp(-\Omega(\eta^2 d)) \leq 0.01$, any given vector in $\bN$ that is $(\cH_b, 2\eta)$-good is also $(\bF_b \setminus \cH_b, 4\eta)$-good. Since $\bu$ is $(\cH_{b}, \eta)$-good, then its nearest net point $\bu'$ is $(\cH_b, 2\eta)$-good. Thus $\bu'$ is $(\bF_b \setminus \cH_b, 4\eta)$-good which implies that $\bu$ is  $(\bF_b \setminus \cH_b, 5\eta)$-good except w.p. $\exp(-\Omega(\eta^2 d)) \leq 0.01$.
But by definition of $\bv$, it perfectly agrees with the marked columns of $\bB_{\bs} \cdot \bP_{\bs}^T$. Since $5\gamma < 0.1$ the above implies that $\sign(\bu)^{\bF_b \setminus \cH_b}$ strongly-agrees (\cref{def:strongly-agree}) with the matrix $\paren{\bB_{\bs} \cdot \bP_{\bs}^T}^{\bF_b \setminus \cH_b}$ which implies that $\sign(\bu)^{[d] \setminus \cH_b}$ strongly-agrees with the matrix $\paren{\bB_{\bs} \cdot \bP_{\bs}^T}^{[d] \setminus \cH_b}$, as required.

\end{proof}

We now ready to prove the final claim that concludes the proof of \cref{lemma:LB:weak}.

\begin{claim}\label{claim:LB:final}
	It holds that
	\begin{align*}
			\ppr{r,r' \la \zo^m, \: X \sim \cD}{\Ac^{\Mc_r,\Fc,\Gc_{r'}}(X)\text{ is strongly-correlated with }X} \geq \frac{0.8 \beta}{k}.
	\end{align*}
\end{claim}
\begin{proof}
	In the following we assume that the event from the statement of \cref{claim:u_strongly_agrees} occurs, and let $\cH = \cH_1 \cup \cH_{-1}$.
	Define the permutation matrix $\bP' \in \cP_{d_0 + \ell}$ that is obtained by removing the rows $\cH$ and the columns $\bP_{\bs}(\cH)$ from $\bP_{\bs}$ (i.e., $\bP'$ is the permutation induced by $\bP_{\bs}$ between $[d]\setminus \cH$ and $[d] \setminus \bP_{\bs}(\cH)$).
	Similarly, define the permutation matrix $\overline{\bP}' \in  \cP_{\ell}$ that is obtained by removing the rows $\overline{\cH} = [d]\setminus\cH$ and the columns $\bP_{\bs}(\overline{\cH})$ from $\bP_{\bs}$ (i.e., $\overline{\bP}'$ is the permutation induced by $\bP_{\bs}$ between $\cH$ and $\bP_{\bs}(\cH)$).
	Note that $\bP'$ is distributed uniformly over $\cP_{d_0 + \ell}$ for any choice of $\overline{\bP}'$. 
	In the following, let $\btPi' = \btPi^{[d] \setminus \cH}$,
	$\: \bq' = \sign(\bu)^{[d] \setminus \cH} \cdot \bP'$, and $\: \bB' = \bB_{\bs}^{[d] \setminus \cH}$. 
	By \cref{claim:u_strongly_agrees} it holds that
	\begin{align}\label{eq:bqbB-strong-agreement}
		\pr{\bq' \text{ strongly-agrees with }\bB'} \geq \frac{0.8 \beta}{k}.
	\end{align}
	But note that $\bB' = \PAP_{n_0, d_0, \ell/2}(\bX, \bP')$ and also note that $\bq'$ is just a function of $\btPi'$ and $\overline{\bP'}$ (i.e., independent of $\bP'$) since it equals to $\sign(\bw)^{\overline{\bP}'([d]\setminus \cH)}$ where $\bw$ is the vector in $Span(\btPi)$ that maximizes $\min\set{\sum_{j\in \overline{\bP}'(\cH_1)} \sign(\bw^j), \: -\sum_{j \in \overline{\bP}'(\cH_{-1})} \sign(\bw^j)}$.
	Furthermore, note that $(\bq' \cdot (\bP')^T)^{1,\ldots,d_0} = \bq$, where $\bq$ is the final output of $\Ac^{\Mc,\Fc,\Gc}(\bX) = \Fc(\bz, \Mc(\Gc(\bX,\bz)))$. 
	Thus by \cref{lemma:PAP,eq:bqbB-strong-agreement} we conclude that
	\begin{align*}
		\ppr{r, r' \la \zo^m, \: X \sim \cD}{\Ac^{\Mc_r,\Fc,\Gc_{r'}}(X)\text{ is strongly-correlated with }X} \geq  \frac{0.8 \beta}{k}.
	\end{align*}
\end{proof}

\subsection{Strong Estimators}\label{sec:LB:strong}

\begin{theorem}[Restatement of \cref{thm:intro:lower_bound:strong}]\label{thm:lower_bound:strong}
	Let $n,k,d \in \bbN$, $\lambda \geq 1$, $\beta \in (0,1]$ such that $d \geq c k$ and $\lambda^2 \leq \frac{d}{c \log k}$ for large enough constant $c > 0$, and $n$ is a multiple of $k$.
	If $\Mc \colon (\cS_d)^n \times [0,1] \rightarrow \bbR^{d \times d}$ is an \emph{$(k, \lambda, \beta, \gamma_{\max} = \frac1{10^6 \lambda^2})$-strong subspace estimator} and $\Mc(\cdot, \gamma)$ is $\paren{1,\frac{\beta}{5nk}}$-DP for every $\gamma \in [0,1]$, then $n \geq \Omega\paren{\frac{k\sqrt{d}/\lambda}{\log^{1.5}(\frac{dk}{\lambda \beta}) \sqrt{\log(d k) \log(2n/k)}}}$.
\end{theorem}

We prove \cref{thm:lower_bound:strong} using a similar technical lemma to \cref{lemma:LB:weak}, but now since $\Mc$ is a strong subspace estimator, we can use $\alpha = \tilde{\Theta}\paren{\frac1{\lambda^2}}$ (rather than $\Theta\paren{\frac1{\lambda^2 k}}$ as in \cref{lemma:LB:weak}).

\begin{lemma}\label{lemma:LB:strong}
	There exists large enough constant $c>0$ such that the following holds:
	Let $n,k,d,\lambda,\beta$ and $\Mc$ as in \cref{thm:lower_bound:strong}, let $\alpha = \frac1{c\cdot  \log(d k) \log(2n/k) \cdot   \lambda^2}$, $\: n_0 = n/k$, $\: \ell = 2\cdot \ceil{\frac14(1-\alpha) d}$, $\; d_0 = d - 2\ell$, $\: \cX = \cS_d$, $\: \cZ = [k] \times (\cP_d)^k$ and $\: \cW = \cW_{d,k}$. Let $\Gc\colon \oo^{n_0 \times d_0} \times \cZ \rightarrow \cX^n$ be \cref{alg:G}, and let $\Fc \colon \cZ \times \cW \rightarrow [-1,1]^{d_0}$ be \cref{alg:F}. Then the triplet $(\Mc, \Fc, \Gc)$ is $\frac{0.8 \beta}{k}$-leaking (\cref{def:beta-accurate}).
\end{lemma}

By \cref{lemma:LB:strong,lemma:framework}, it holds that $n_0 \geq \Omega\paren{\frac{\sqrt{d_0}}{\log^{1.5}(d_0/\beta)}}$. The proof of \cref{thm:lower_bound:strong} now follows since $n_0 = n/k$ and $d_0 \approx \alpha d$ for the $\alpha$ defined in the lemma.

\subsubsection{Proving \cref{lemma:LB:strong}}

As in the proof of \cref{lemma:LB:weak}, we define random variables $\bX \sim \cD(n_0, d_0)$ (\cref{def:D}) and $\bz = (\bs, (\bP_1,\ldots,\bP_k)) \la \cZ$, and consider a random execution of $\Ac^{\Mc,\Fc,\Gc}(\bX) = \Fc(\bz, \Mc(\Gc(\bX, \bz)))$. Let $\set{\bA_t}, \set{\bB_t}, \bY$ be the values of $\set{A_t},\set{B_t}, Y\in (\cS_d)^n$ in the  execution of $\Gc$, and recall that $\bY$ is a vertical concatenation of $\bY_1, \ldots, \bY_k$ where $\bY_t = \frac1{\sqrt{d}} \bB_t$.

The only difference from proving \cref{lemma:LB:weak} is to prove a different version of \cref{claim:large_spec_gap:weak} that only considers the gap between $\sigma_{k}$ and $\sigma_{k+1}$ (which will meet the requirements of the strong estimator $\Mc$).
Namely, is suffices to prove the following claim:

\begin{claim}\label{claim:large_spec_gap:strong}
	It holds that
	\begin{align*}
		\pr{\sigma_{k+1}(\bY) \leq \gamma \cdot \sigma_k(\bY)} \geq 0.9,
	\end{align*}
	for $\gamma = \frac1{1000 \lambda}$.
\end{claim}

\begin{proof}[Proof of \cref{claim:large_spec_gap:strong}]
	
	As in \cref{def:vt}, for $t \in [k]$ we define
	
	\begin{align}\label{def:vt:storng}
		\bv_t = \frac1{\sqrt{2\ell}}\cdot (\underbrace{0\dots,0}_{d_0},\underbrace{1,\ldots \ldots \ldots ,1}_{\ell},\underbrace{-1,\ldots \ldots \ldots ,-1}_{\ell})\cdot \bP_t \: \in \: \cS_d
	\end{align}

	Let $\bE = Span\set{\bv_1,\ldots,\bv_k}$. Note that by construction, $\set{\bv_t}, \bE$ are independent of $\set{\bA_t}$.
	
	Similarly to \cref{eq:vs_vt} it holds that 
	
	\begin{align}\label{eq:vs_vt:strong}
		\pr{\forall s,t \in [k]: \abs{\ip{\bv_s, \bv_t}} \leq O\paren{\sqrt{\frac{\log k}{d}}}} \geq 0.99.
	\end{align} 

	This yields (using similar steps as in the proof of \cref{claim:large_spec_gap:weak}) that
	w.p. $0.98$, $\dim(\bE) = k$ and every unit vector $\bu \in \bE$ has $\norm{\bY \bu}_2^2 \geq (1 - 4\alpha)\frac{n}{k}$. This in particular implies that $\sigma_k^2(\bY) \geq (1 - 4\alpha)\frac{n}{k}$.

	Our goal is to prove that  w.p. $0.99$ it also holds that $\sigma_{k+1}^2(\bY) \leq \tilde{O}(\alpha) \frac{n}{k}$.
	Let $\bar{\bE}$ be the orthogonal subspace to $\bE$.
	Our goal is reduced to showing that there exists a constant $c$ such that
	\begin{align}\label{eq:LB:strong:goal}
		\pr{\forall u \in \bar{\bE} \cap \cS_d: \quad \norm{\bY u}_2^2 \leq \alpha \cdot  c  \log(k/\alpha) \log(2n/k)\cdot  \frac{n}{k}} \geq 0.99.
	\end{align}
	Given that the event in \cref{eq:LB:strong:goal} holds we conclude that $\sigma_{k+1}^2(\bY) \leq \alpha \cdot c \log(k/\alpha) \log(2n/k)\cdot \frac{n}{k}$ and hence
	\begin{align*}
		\frac{\sigma_{k+1}^2(\bY)}{\sigma_k^2(\bY)} \leq \frac{\alpha \cdot c \log(k/\alpha) \log(2n/k) \cdot \frac{n}{k}}{(1-4\alpha)\cdot \frac{n}{k}} \leq \frac1{10^6 \lambda ^2} = \gamma^2,
	\end{align*}
	where the last inequality holds by taking
	\begin{align*}%\label{eq:alpha}
		\alpha = \frac{1}{2\cdot 10^6 \cdot  c \log(k/\alpha) \cdot \log(2n/k) \cdot \lambda^2}.
	\end{align*}
	Note that for every $\bu \in \bar{\bE}\cap \cS_d$ is holds that 
	\begin{align*}
		\norm{\bY \bu}_2^2 =\sum_{t=1}^k \norm{\bY_t \bu}_2^2 = \frac1{d} \sum_{t=1}^k \norm{\bA_t \bu}_2^2,
	\end{align*}
	where the last equality holds since $\bu$ is orthogonal to $\bv_1,\ldots,\bv_k$.
	Therefore, we can prove \cref{eq:LB:strong:goal} by proving that there exists a constant $c'$ such that
	\begin{align}\label{eq:LB:strong:second-goal}
		\forall u \in \cS_d: \quad \pr{\frac1{d}\cdot \sum_{t=1}^k \norm{\bA_t u}_2^2 > \alpha \cdot c' \log(k/\alpha) \log(2n/k) \cdot \frac{n}{k}} \leq \exp(-d \ln(3k/\alpha) - 10).
	\end{align}
	Given that \cref{eq:LB:strong:second-goal} holds, we prove the claim using a net argument.
	By taking an $\exp\paren{d \ln(3k/\alpha)}$-size $\sqrt{\alpha/k}$-net of $\cS_{d}$ (\cref{fact:net}), \cref{eq:LB:strong:goal} follows by \cref{eq:LB:strong:second-goal} and the union bound over all the net points, which concludes the proof of the claim.
	
	In the following we focus on proving \cref{eq:LB:strong:second-goal}.
	Fix a columns vector $u \in \cS_d$. 
	Recall that $\bA_t \sim \cD(n_0,d_0)$ (\cref{def:D}), and $\bA_t$ is located in $d_0$ random columns out of $d$, which are the columns $\bJ = \bP_t([d_0])$.
	 Let $\ba_{t,i} \in \set{\pm \frac1{\sqrt{d}}}^{d_0}$ be the \ith row of $\bA_t$. By \cref{def:D}, the coordinates of $\ba_{t,i}$ are i.i.d. Bernoulli distribution over $\oo$, each takes $1$ w.p. $1/2$ and $-1$ w.p. $1/2$. Therefore by \cref{fact:ip-concentration} it holds that
	\begin{align*}
		\forall \xi \geq 0: \quad  \pr{\ip{\ba_{i,t}, u}^2 \geq \xi}
		&= \eex{\bJ}{\pr{\ip{\ba_{i,t}, u_{\bJ}}^2 \geq \xi}}
		\leq \eex{\bJ}{2 \exp\paren{-\frac{d \xi}{2\norm{u_{\bJ}}_2^2}}}\\
		&\leq 2 \exp\paren{-\frac{d \xi}{\ex{2\norm{u_{\bJ}}_2^2}}}
		\leq 2 \exp\paren{-\frac{d \xi}{2 \alpha}},
	\end{align*}
	where the first inequality holds by  \cref{fact:ip-concentration}, the second one holds by Jensen's inequality (and since the function $f(x) = e^{-1/x}$ is concave), and the last one hold since $\ex{\norm{u_{\bJ}}_2^2} = \frac{d_0}{d} \leq \alpha$. 
	By the union bound over the $n/k$ rows of $\bA_t$ we obtain that
	\begin{align*}
		\forall \xi \geq 0: \quad \pr{\norm{\bA_t w}_2^2 \geq \frac{\xi n}{k}} \leq \frac{2n}{k}\cdot \exp\paren{-\frac{d \xi}{2\alpha}},
	\end{align*}
	or equivalently 
	\begin{align}\label{eq:almost-sub-exp}
		\forall \xi \geq 0: \quad \pr{\norm{\bY_t u}_2^2 \geq \xi} = \pr{\norm{\bA_t w}_2^2 \geq \xi} \leq \frac{2n}{k}\cdot \exp\paren{-\frac{d k}{2 \alpha n}\cdot \xi}.
	\end{align}

	In the following, let $\bb_t =  \norm{\bY_t u}_2^2$, and define $\bb_t' = \bb_t - \ex{\bb_t}$, and $\mu = \frac{2\alpha n \ln(2n/k)}{dk}$.
	First, note that 
	\begin{align*}
		\ex{\bb_t} 
		&= \int_{0}^{\infty} \pr{\bb_t > \xi} d\xi = \mu + \int_{\mu}^{\infty} \pr{\bb_t > \xi}d\xi\\
		&\leq \mu + \int_{\mu}^{\infty} \frac{2n}{k}\cdot \exp\paren{-\frac{d k}{2 \alpha n}\cdot \xi}d\xi\\
		&= \mu + \underbrace{\left[-\frac{2n}{k}\cdot \frac{2 \alpha n}{dk}\cdot \exp\paren{-\frac{dk}{2 \alpha n}\cdot \xi}\right]_{\mu}^{\infty}}_{\leq \frac{2\alpha n}{dk}}\\
		&\leq 2\mu.
	\end{align*}
	In addition, it holds that 
	\begin{align*}
		\ex{\exp\paren{\frac{\size{\bb'_t}}{8\mu}}}
		&= \int_{0}^{\infty} \pr{\exp\paren{\frac{\size{\bb'_t}}{8\mu}} > \xi} d\xi\\
		&\leq 3/2 +  \int_{3/2}^{\infty} \pr{\size{\bb'_t} > 8\mu \ln(\xi)} d\xi\\
		&= 3/2 + \int_{3/2}^{\infty} \paren{\pr{\bb_t > 8\mu \ln(\xi) + \ex{\bb_t}} + \pr{\bb_t < -8\mu \ln(\xi) + \ex{\bb_t}}}d\xi\\
		&\leq 3/2 + \int_{3/2}^{\infty} \paren{\pr{\bb_t' > 8\mu \ln(\xi)} + \underbrace{\pr{\bb_t < -8\mu \ln(3/2) + 2\mu}}_0}d\xi\\
		&\leq 3/2+ \int_{3/2}^{\infty} \frac{2n}{k}\cdot \exp\paren{-\frac{d k}{2\alpha n}\cdot 8\mu \ln(\xi)} d\xi\\
		&= 3/2 + \int_{3/2}^{\infty} \xi^{-8}d\xi\\
		&\leq 2.
	\end{align*}
	Namely, $\bb'_t$ is a Sub-Exponential random variable (\cref{def:sub-exp}) with $\norm{\bb'_t}_{\psi_1} \leq 8 \mu$.
	Since $\set{\bb'_t}$ are independent, each has zero mean, we obtain by \cref{fact:sub-exp-concent} that
	\begin{align*}
		\forall \xi \geq 0 : \quad \pr{\sum_{t=1}^k \bb_t \geq 2k\mu + \xi}
		\leq \pr{\sum_{t=1}^k \bb_t' \geq \xi}
		\leq 2\exp\paren{-\Omega\paren{ \min\paren{\frac{\xi^2}{k 64\mu^2}, \: \frac{\xi}{8\mu}}}}.
	\end{align*}
	We now take $\xi = c \cdot \mu d \ln(k/\alpha) \geq c \alpha\cdot \log(k/\alpha)\log(2n/k)\cdot \frac{n}{k}$ for large enough constant $c$.
	Recall that by our assumption on $d$ it holds that $\xi \geq 4k\mu$.
	Hence
	\begin{align*}
		\pr{\sum_{t=1}^k \norm{\bY_t u_t}_2^2 \geq \xi}
		&\leq  \pr{\sum_{t=1}^k \bb_t \geq 2k\mu + \xi/2}\\
		&\leq \exp\paren{-15 d \cdot \ln(k/\alpha)},\\
		&\leq \exp\paren{-d \cdot \ln(3k/\alpha) - 10},
	\end{align*}
	where the second inequality holds assuming that $c$ is large enough. 
	This concludes the proof of \cref{eq:LB:strong:second-goal} and therefore the proof of the lemma.
\end{proof}

\section{Empirical Evaluation} \label{sec:experiments}

We implemented a zCDP (\cref{def:zCDP}) variant of our subspace estimation algorithm in Python (denoted by $\EstSubspace$), and in this section we present empirical results for the fundamental task of privately estimating the average of $d$ dimensional points that approximately lie in a much smaller $k$-dimensional subspace. Namely, given a dataset $X = (x_1,\ldots,x_n) \in \cS_d^n$, a parameter $k$, and zCDP parameters $\rho, \delta$, we perform the following steps: (a) Compute a $(\rho/2,\delta)$-zCDP rank-$k$ projection matrix $\tilde{\Pi}$ using $\EstSubspace$ that estimates the projection onto the top-$k$ rows subspace of $X$, (b) Compute a $\rho/2$-zCDP estimation of the average of $X$ using the Gaussian Mechanism:  $\:\tilde{x} = \frac1n \sum_{i=1}^n x_i + \cN(\pt{0}, \sigma^2\cdot \bbI_{d \times d})$ for $\sigma =  \frac{2}{n \sqrt{\rho}}$, and (c) Output $\hat{x} = \tilde{\Pi} \cdot \tilde{x}$. 

\remove{
\begin{enumerate}[a.]
	\item Compute a $(\\rho/2, \delta)$-zCDP rank-$k$ projection matrix $\tilde{\Pi}$ using $\EstSubspace$ that estimates the projection onto the top-$k$ rows subspace of $X$.\label{step:exp:Pi}
	
	\item Compute a $\rho/2$-zCDP estimation of the average of $X$ using the Gaussian Mechanism w.r.t. $\ell_2$-sensitivity $2/n$. I.e., compute $\tilde{x} = \frac1n \sum_{i=1}^n x_i + \cN(\pt{0}, \sigma^2\cdot \bbI_{d \times d})$ for $\sigma =  \frac{2}{n \sqrt{\rho}}$.\label{step:exp:Gaus}
	
	\item Output $\hat{x} = \tilde{\Pi} \cdot \tilde{x}$.\label{step:exp:output}
\end{enumerate}
}

The accuracy is measured by the $\ell_2$ error from the average: $\norm{\hat{x} -  \frac1n \sum_{i=1}^n x_i}_2$.

In all our experiments, we use $\rho = 2$ and $\delta = 10^{-5}$,  $t = 125$ (the number of subsets in the sample-and-aggregate process), $n = 2tk$ data points, $q = 10\cdot k$ (the number of reference points in the aggregation), and use the zCDP implementation of the FriendlyCore-based averaging algorithm of \cite{FriendlyCore22}.\footnote{Their source code is publicly available at  \url{https://media.icml.cc/Conferences/ICML2022/supplementary/tsfadia22a-supp.zip}.} All experiments were tested on a MacBook Pro Laptop with 8-core Apple M1 CPU with 16GB RAM.

Rather than using \cite{FriendlyCore22}'s algorithm for the known-diameter case, we use their unknown-diameter implementation with $\xi_{\min} = 10^{-6}$ and $\xi_{\max} = 100$ (see \cref{remark:unknown_gamma_k} for details).
Furthermore, we reduced the space complexity of our implementation from $\tilde{\Theta}(d^2)$ to $\tilde{\Theta}(k d)$.\footnote{We do not explicitly compute a $d\times d$ rank-$k$ projection matrix in each subset, but rather only compute a good approximation of the top-$k$ rows $V = (v_1,\ldots,v_k) \in \cS_d^k$ using the Python function \textsf{randomized\_svd} (provided in the \textsf{sklearn} library). We then compute the projection of any vector $u \in \bbR^d$ onto $Span\set{v_1,\ldots,v_k}$, given by $V^T V u$, from right to left, which only involves $O(kd)$ time and space computation cost. We do the same thing w.r.t. to the output projection $\tilde{\Pi}$ (i.e., represent it using only $k$ vectors).}

In order to generate a synthetic dataset that approximately lie in a $k$-dimensional subspace, we initially sample uniformly random $b_1,\ldots,b_k \la \oo^d$ and perform the following process to generate each data point: (i) Sample a random unit vector $u$ in $Span\set{b_1,\ldots,b_k}$, (ii) Sample a random noise vector $\nu \la \set{1/\tau, -1/\tau}^d$, and (iii) Output $\frac{u + \nu}{\norm{u + \nu}}$ (note that higher $\tau$ results with data points that are closer to a $k$-dimensional subspace).

We compare our averaging method to two other approaches: The first one simply applies the Gaussian mechanism directly on $X = (x_1,\ldots,x_n)$ using the entire privacy budget $\rho$ (i.e., without computing a projection matrix). The second one replaces our Step (a) by computing the projection matrix $\tilde{\Pi}$ using a $(\rho/2,\delta)$-zCDP variant of the additive-gap based algorithm of \cite{DTTZ14} (see \cref{sec:prelim:DworkSubspace} for more details).
\footnote{We remark that unlike $\EstSubspace$, \cite{DTTZ14}'s algorithm requires $O(d^2)$ time and space complexity as it requires an explicit access to the $d \times d$ projection matrix onto the top-$k$ rows subspace, and therefore is limited to moderate values of $d$. Still, we were able to use it as baseline since we saw the advantage of our approach in terms of accuracy even when $d$ is not extremely high.}
%The algorithm of \cite{GGB18} can be implemented with a better space complexity. But since its accuracy guarantee is not better than \cite{DTTZ14}, we decided to stick with \cite{DTTZ14} since we could still see the advantage of our approach in terms of accuracy even when $d$ is not extremely high.}
The empirical results are presented in  \cref{figures}. In all experiments, we perform $30$ repetitions for generating each graph point which represents the trimmed average of values between the $0.1$ and $0.9$ quantiles.
We show the $\ell_2$ error of our estimate on the $Y$-axis.
The first graph illustrates the inherent dependency on $d$ that \cite{DTTZ14}'s algorithm has, while our algorithm $\EstSubspace$ takes advantage of the closeness of the points to dimension $k$ in order to eliminate this dependency. The second graph illustrates that when $d$ is fixed, increasing $k$ and $n$ in the same rate a has similar affect on both $\EstSubspace$ and \cite{DTTZ14}'s algorithm. In the last graph we compare the accuracy of $\EstSubspace$ and \cite{DTTZ14}'s algorithm as a function of the closeness to a subspace $k$ (measured in our experiments by the parameter $\tau$), and show in what regimes $\EstSubspace$ outperforms \cite{DTTZ14}'s algorithm.

\begin{figure*}
		\centerline{
			\includegraphics[scale=.35]{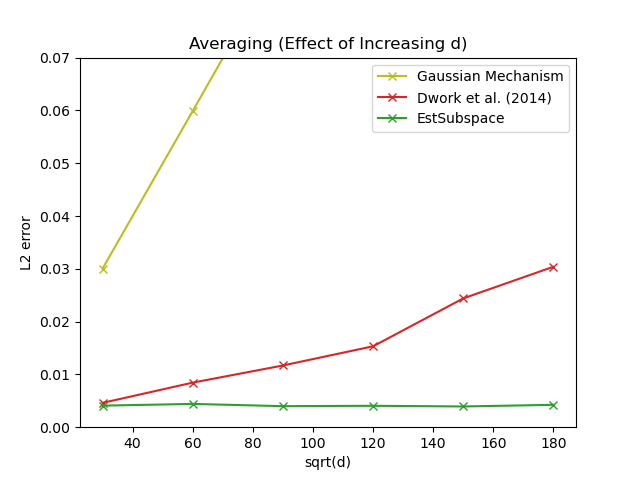}
			\includegraphics[scale=.35]{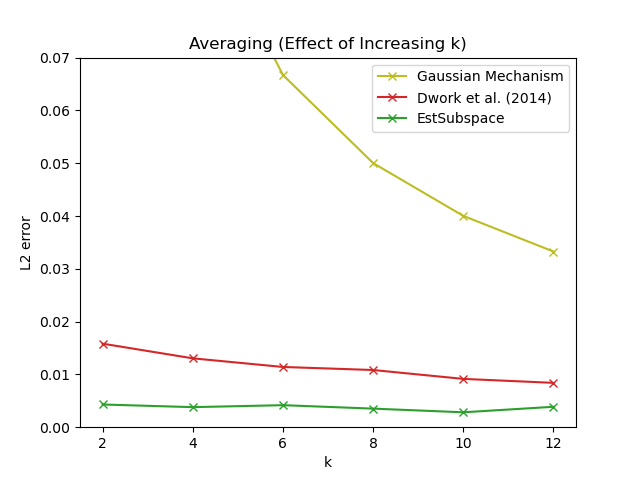}
			\includegraphics[scale=.35]{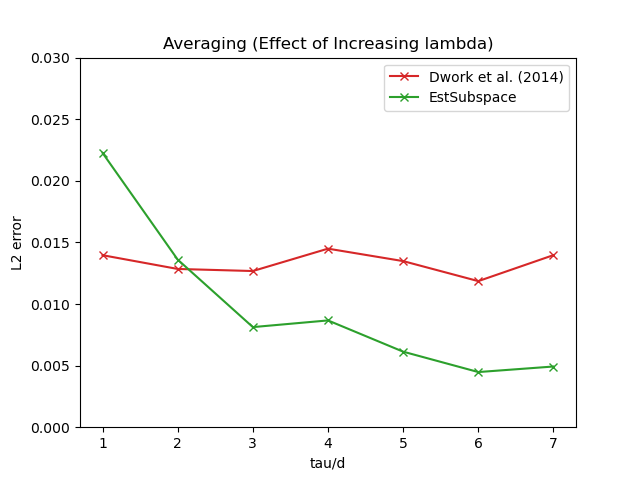}		
		}
		\caption{From Left to Right: (1) The case $k=4$ and $\tau=10 d$, varying $d$ (the $X$-axis is $\sqrt{d}$). (2) The case $d = 10^4$ and $\tau=10 d$, varying $k$. (3) The case $d = 10^4$ and $k = 4$, varying $\tau$ (the $X$-axis is $\tau/d$). In all the experiments, we use $n = 250\cdot k$ data points.}
		\label{figures}
\end{figure*}

\ifdefined\IsAnonymous
\else
\section*{Acknowledgments}

The author would like to thank Edith Cohen and Jonathan Ullman for very useful discussions.

\fi

%\addcontentsline{References}
\printbibliography

%\appendix

\end{document}